\documentclass[cmbright]{staauth}

\usepackage{moreverb}

\usepackage[colorlinks,bookmarksopen,bookmarksnumbered,citecolor=red,urlcolor=red]{hyperref}

\usepackage{natbib}

\usepackage{float,amsmath,amsthm,verbatim,multirow}
\usepackage{graphicx}
\usepackage{subfig}

\newtheorem{theorem}{Theorem}

\newtheorem{corollary}{Corollary}

\theoremstyle{definition}

\newcommand\BibTeX{{\rmfamily B\kern-.05em \textsc{i\kern-.025em b}\kern-.08em
		T\kern-.1667em\lower.7ex\hbox{E}\kern-.125emX}}

\newcommand{\bX}{{\mathbf X}}

\newcommand{\bI}{{\bf I}}
\newcommand{\bSigma}{{\mathbf{\Sigma}}}

\newcommand{\bmu}{{\mbox{\boldmath${\mu}$}}}
\newcommand{\mC}{{\mathcal{C}}}

\newcommand{\mD}{{\mathcal{D}}}

\newcommand{\bG}{{\mathbf{\Gamma}}}
\newcommand{\bL}{{\mathbf{\Lambda}}}
\newcommand{\bxi}{{\mbox{\boldmath${\xi}$}}}
\newcommand{\bgamma}{{\mbox{\boldmath${\gamma}$}}}
\newcommand{\mR}{{\mathcal{R}}}

\newcommand{\norm}[1]{\left\lVert#1\right\rVert}

\def\volumeyear{2017}

\runninghead{A D Peterson, A P Ghosh and R Maitra}{ $K$-means
  Hierarchical Cluster Merging}

\begin{document}
\title{Merging $K$-means with hierarchical clustering for 
  identifying general-shaped groups}
\author{Anna D. Peterson\affil{a}, Arka P. Ghosh\affil{a} and
  Ranjan Maitra\affil{a}\corrauth}
\address{%
\affilnum{a}
Department of  Statistics, Iowa State University, Ames, Iowa, USA}

\corremail{maitra@iastate.edu}

\received{16 November 2017}
\accepted{28 November 2017}

\begin{abstract}
Clustering partitions a dataset 
such that observations placed together in a group are similar but 
different from those in other groups. Hierarchical and
$K$-means clustering are two approaches but have different strengths
and weaknesses.  For instance, hierarchical clustering identifies
groups in a tree-like structure but suffers from computational
complexity in large datasets while  $K$-means clustering 
is efficient but  designed to identify homogeneous
spherically-shaped clusters.  We present a
hybrid non-parametric clustering approach that amalgamates the two
methods  to identify
general-shaped clusters and that can be applied to larger
datasets. Specifically, we first 
partition the dataset into spherical groups using $K$-means.
We next merge these groups using hierarchical methods with a
data-driven distance measure as a stopping criterion.  Our proposal
has the potential to reveal groups with general shapes and structure
in a dataset. We demonstrate good performance on several simulated
and real datasets.  
\end{abstract}
\maketitle

\section{Introduction} \label{intro}

Clustering partitions a dataset into subsets called clusters without any
prior knowledge of group assignment.   The 
general objective is that observations placed in the same cluster are
similar in some sense while being different to those in other groups.
The substantial body of literature
\citep{everittetal01,fraleyandraftery02,hartigan85,kaufmanandrousseuw90,kettenring06,melnykovandmaitra11,mclachlanandbasford88,murtagh85,ramey85}
dedicated to the topic reflects the difficulty and diversity of
clustering applications. Most unsupervised clustering techniques are
broadly hierarchical or partition-optimization-based.  Traditionally,
hierarchical algorithms  provide a tree-like
structure for demarcating groups, with the property that all
observations in a group at some branch are also in the same group higher up the
tree. 
Hierarchical algorithms may be
agglomerative (cluster-merging) or divisive (cluster-breaking).
Agglomerative algorithms successively merge smaller clusters together
whereas divisive algorithms successively break larger clusters apart.
Most hierarchical clustering methods use some  dissimilarity measure
between groups to decide whether to merge (or split)
groups. The result can be represented as a dendrogram that can
visually express the data structure.  Generally, a
linkage criteria specifies the dissimilarity between each branch of
the dendrogram as a function of the pairwise distances of observations
in the sets.  The linkage criterion can influence cluster shapes: for
example, single linkage is commonly associated with
stringy groups while Ward's linkage is more commonly used for
spherical clusters \citep{johnston07}. Although the nesting structure
provides a broad understanding of the relationships between
observations within a dataset, clusters lose homogeneity at higher
branches of the tree.  Further, hierarchical clustering requires
calculating all pairwise distances between observations which is 
computationally expensive in 
processor speed (and, more so, in memory) for larger datasets. 

Partitional clustering, on the other hand, directly
divides a dataset into groups, so that the
data in each subset (ideally) share some common trait.
Typically the algorithm involves minimizing some measure of
dissimilarity between observations within each cluster, while maximizing the
dissimilarity between observations in different clusters. The
$K$-means algorithm is a very popular choice even though more formal
approaches are provided by model-based clustering~\citep{fraleyandraftery02,mclachlanandpeel00,melnykovandmaitra11}.  The $K$-means algorithms 
minimizes the within group sum-of-squares and can be implemented
efficiently~\citep{hartiganandwong79}. But $K$-means requires the
number of groups ($K$) to be provided or alternatively decided from the
data~\citep{maitraetal12}. Further, different initialization 
strategies often produce strikingly different 
groupings. Also, the algorithm is not as successful with groups that
do not have the same and spherical dispersion structure.

\begin{figure*}[!h]
  \centering
  \hspace{-0.1in}
  \mbox{
    \hspace{-0.1in}
    \subfloat[$K$-means, $K$ = 2]{
      \includegraphics[totalheight=2.2in,width=2.2in]{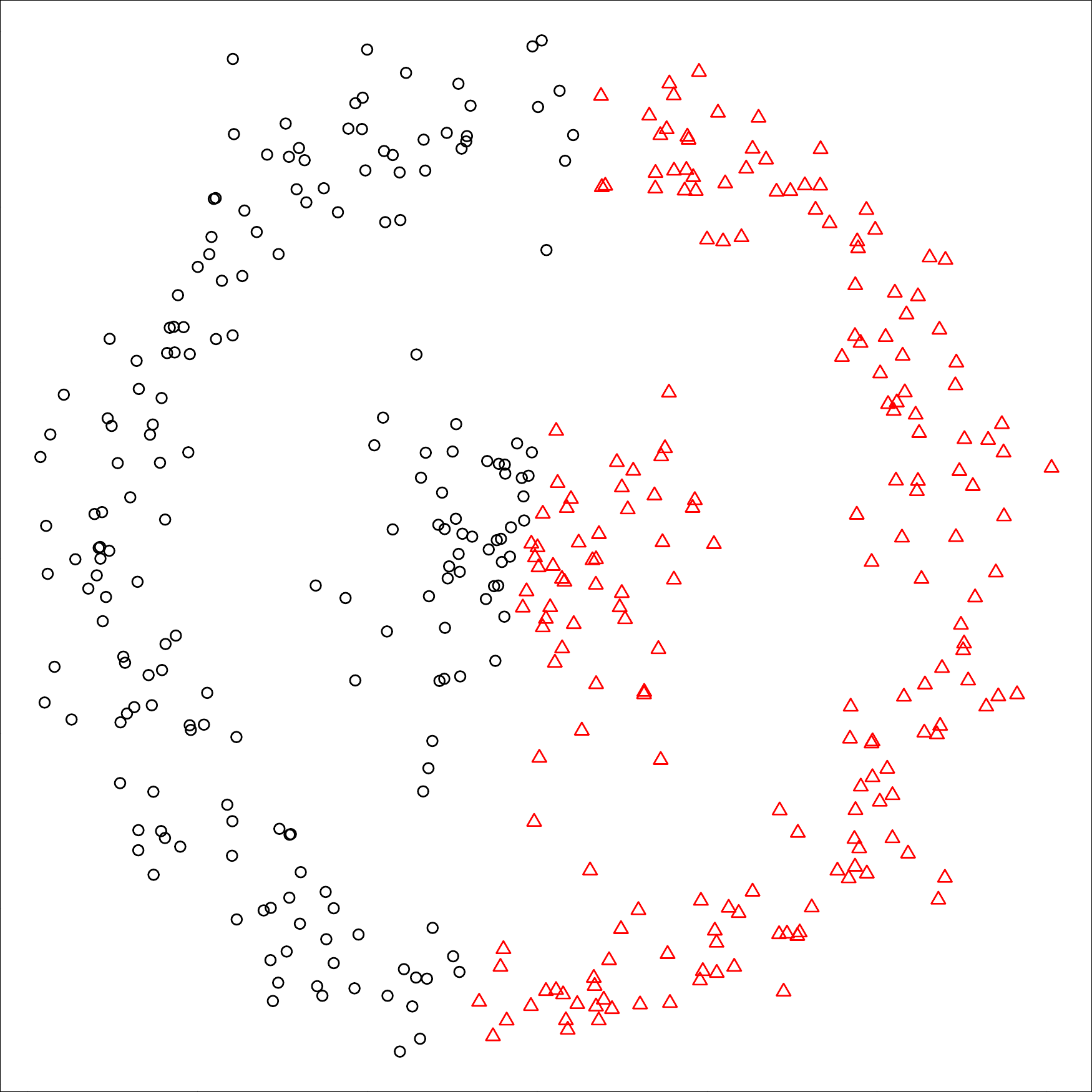}}
    \subfloat[$K$-means, $K$ = 6]{
      \includegraphics[angle=0,totalheight=2.2in,width=2.2in]{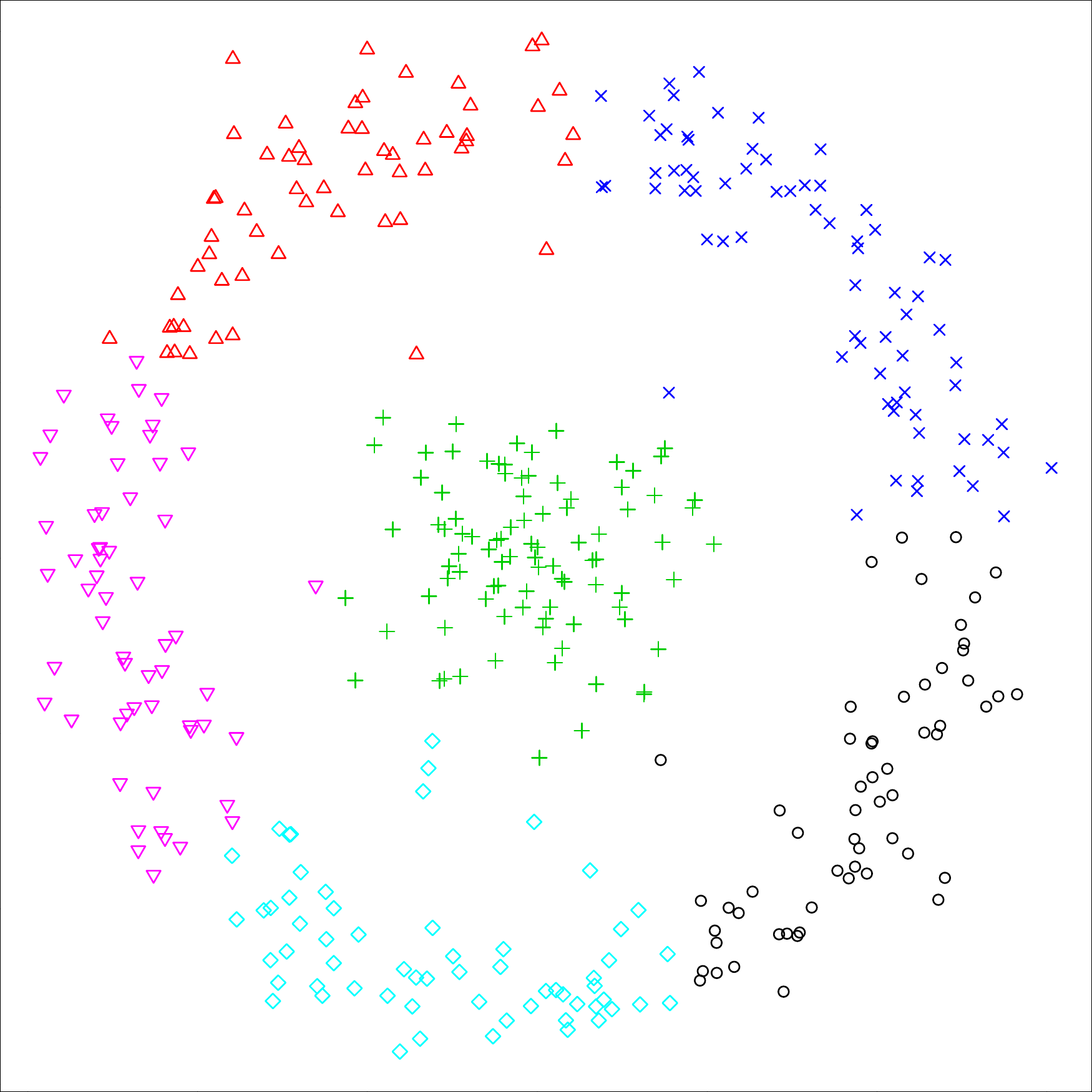}}
    \subfloat[HC, $K$ = 2, single linkage]{
      \includegraphics[angle=0,totalheight=2.2in,width=2.2in]{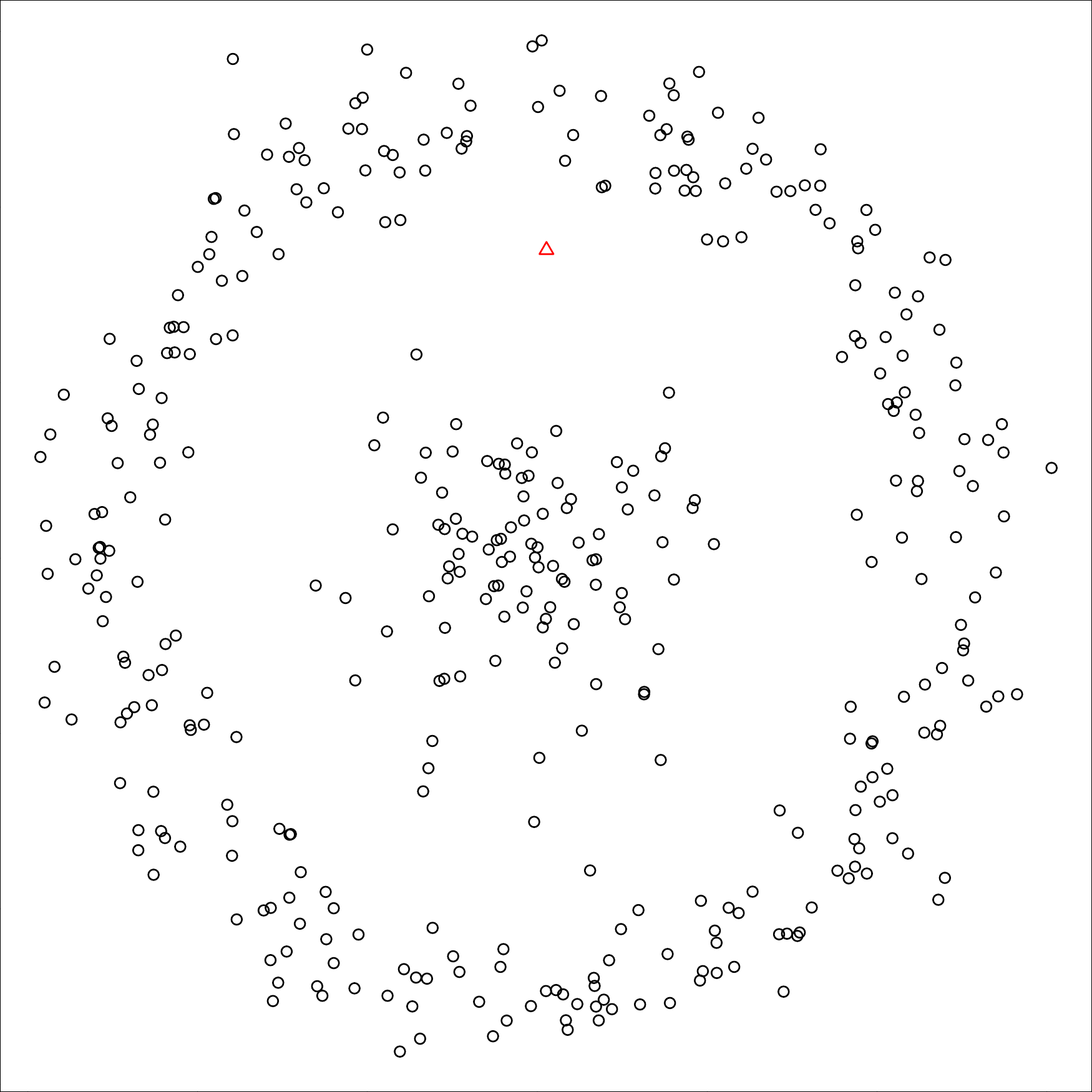}}}
  \caption{Partition (using both color and symbol) of the Bullseye
    dataset with (a) 2-means, (b) 6-means and (c) single linkage
    hierarchical clustering with 2 groups.   } 
  \label{bull}
\end{figure*} 
We illustrate some shortcomings of these algorithms through the
Bullseye dataset of~\citet{nugent10} which has 400 observations from a
spherical cluster surrounded by a 
ring of observations (which form the second group).
Figure \ref{bull}(a-b) shows the clustering using 2- and $6$-means.
In addition, Figure \ref{bull}(c) shows the grouping 
based on hierarchical clustering with single linkage and $K$ = 2.
Neither approach clusters into their true groupings. Although Figure
\ref{bull}(b) captures the center group, we required 5 groups to
create the outer ring. One possibility of improving this solution is
to merge these groups using some objective mechanism and we will
explore this approach in this paper. 


The idea of merging clusters is not new in the literature.
\citet{fred05} introduced evidence accumulation clustering (EAC) for
combining the results from multiple applications of $K$-means.  The
idea behind EAC is that each partition gives 
independent evidence on the organization of the data. The authors
proposed independent runs of $K$-means on the dataset and created a
similarity (frequency) matrix between all pairs of data points with
the $(i,j)$th entry representing the number of times the $i$th and
$j$th observations were placed in the same group.  The
final data partition is  obtained by applying a hierarchical
agglomerative clustering algorithm using this similarity matrix.  The
motivation here is that observations that are together in
the majority of partitions should also be so in the final
chosen partition.  This procedure is novel in that it chooses among
several different partitions but it is computationally expensive since
it involves performing either single linkage or average
linkage on an $n\times n$ distance matrix, where $n$ is the number of
observations.  \citet{nugent10} adopt a nonparametric approach to clustering based on the premise that groups correspond to modes of the density.  \citet{nugent10} find the modes within a dataset and assign observations to the ``domain of attraction" of a mode.  The collection of high density modes is used to create a hierarchical structure where dissimilarity between modes is based on the lowest density observed between any pair of groups.  
\citet{baudryetal08} propose a cluster merging method using a
model-based clustering approach. They propose first selecting the
total number of Gaussian mixtures components, $K_0$, using BIC and
then combining them hierarchically.  This yields a unique soft
clustering for each $K$ less than $K_0$. Further refinements to this
method were provided by the DEMP~\citep{hennig10} and 
DEMP+~\citep{melnykov16} algorithms. However, model-based clustering
is computationally slower and typically more difficult to apply on to
larger datasets.  


In this paper, we propose a $K$-means hierarchical  ($K-mH$) 
 cluster merging algorithm which combines the computational benefits
 of  $K$-means with agglomerative hierarchical clustering. The general
 methodology and our algorithm are detailed in Section \ref{method}.
We present several examples of datasets with clusters of 
complicated/general shapes in Section
\ref{examples} to illustrate and evaluate our algorithm.  We end with
a short discussion.


\section{Methodology}
\label{method}
Let ${\mathcal S} =
\{\bX_{1},\bX_{2},\ldots,\bX_{n}\}$ be a dataset of $n$
$p$-dimensional observations that are presumed to be in a partition
$P$ comprising defined categories $\mC_1, \mC_2,\ldots,\mC_K$
according to some similarity measure between observations.  
Suppose we have  $N$ such partitions of a dataset ${\mathcal S}$ where
$\Psi = \{P^1,P^2,\ldots ,P^N\}$ is the set of the $N$ partitions.
Then we define $P^i = \{\mC^i_1, \mC^i_2,\ldots,\mC^i_{K_i}\}$ as a
candidate partition where $\mC^i_j$ is cluster $j$ of partition $P^i$,
$|\mC^i_j|$ is the number of observations in $\mC^i_j$, $K_i$ is the
number of clusters in partition $P^i$ and $\sum_{j=1}^{K_i} |\mC^i_j|
= n $ for all $i$.  Then the goal is to find, among the $N$ partitions
in $\Psi$, the optimal partition $P_{*}$ that ideally provides a close
match to the true partition. Our objective in this paper is to provide
methodology to identify the partitions $P_i$ and the optimal
$P_*$. 
\subsection{Background and Preliminaries} The development of our algorithm borrows ideas from
$K$-means and hierarchical clustering, so we revisit them briefly. 
\subsubsection{$K$-means:} \label{kmeans1}
The $K$-means algorithm starts with $K_0$ $p$-dimensional seeds
$\{\bmu^{(0)}_k; 1\leq k \leq K_0\}$ and then iterates between cluster
assignments and mean updates till convergence. Therefore, at the
$i$th step, we update our partitions to be
 $\mC_k^{(i)} = \{\bX_j \text{ : }
   \|\bX_j - \bmu^{(i)}_k \|\ = \min\limits_{1\leq l \leq
      K}\|\bX_j - \bmu^{(l)}_k \| j = 1,\ldots,n\}$,
  for $k=1,2,\ldots,K_0$, with $\|x \| = \sqrt{x'x}$. These
  updates are followed by recalculated cluster means, with
$\bmu^{(i+1)}_k  = \sum_{j\in \mC_k^{(i)}}\bX_{j}/{|\mC_k^{(i)}|}$. 
The algorithm continues until there are no further changes in  $\{\mC_k^{(i)}:
  k=1, \ldots, K_0\}$ (or, equivalently, in the $\bmu^{(i)}_k$s).
\paragraph{Initialization:}
Initialization can greatly impact performance of
$K$-means~\citep{maitra07} so we adopt \citet{mac67}'s suggestion
that samples $K$ distinct 
observations from the dataset as initial seeds and runs the algorithm
to convergence. We run this procedure $I$ times, with the
converged solution having the smallest within-group sum-of-squares
chosen as our $K$-groups partition. This approach is the default
setting of the {\tt kmeans()} function in R~\citep{R}, with the number of
initializations set by the {\tt nstart} argument. 
\paragraph{Choosing $K_0$:}
\label{diffg} 
Many methods~\citep[for
example,][]{marriott71,tibshiranietal01,mclachlan87,
sugarandjames03,maitraetal12} exist for choosing
$K_0$. Here we discuss the 
\citet{Krzanowski88} criterion which uses the trace of the 
pooled within-group variance-covariance matrix, which we denote as
$W_g$ for a $K$-groups partition. Following \citet{Krzanowski88},
$trace(W_K)$ should decrease dramatically as $K$ increases provided
that $K <\tilde{K}$, where $\tilde K$ is the true number of spherical
groups, but that this decrease should slow down once 
$K\geq\tilde{K}$.  Based on this rationale, and defining 
$Diff(K) = (K-1)^{2/p}trace(W_{K-1}) - K^{2/p}trace(W_{K})$,
the number of homogeneous spherically-dispersed groups $K_0$ can be
obtained as follows: Let $C_K = |Diff(K)/Diff(K+1)|$ and $K_1 , K_2
,\ldots, K_l$ be such that $C_{K_1}\geq C_{K_2}\geq,\ldots,\geq
C_{K_l}$.  Then choose $K_0 = K_1$.

\subsubsection{Agglomerative Hierarchical clustering:} Here, we
successively merge current groups assuming a distance
$d(A,B)$ between any two groups $A$ and $B$ and a 
mechanism (or {\em linkage}) to recalculate the distances when groups are
merged. Examples of linkages are {\em single} where $d(A,B)
= \min\{\|x-y \|: x\in A, y \in 
B\}$ or {\em average} with $d(A,B) = \sum_{x\in
  A}\sum_{y \in B}\|x-y \|/(|A||B|)$. The algorithm initially places
every observation in its own group, that is, by setting 
$\tilde{\mC_j}^{(0)} = \bX_j$ for all $j = 1,2,\ldots,n$.  Then, we
successively merge clusters at each stage, so that at the $i$th
stage, we have $n-i$ clusters, with $(n-i-2)$ many of those groups
unchanged from the previous stage. That is, we have $\tilde{\mC}_j^{(i)} \equiv
\tilde{\mC}_j^{(i-1)}$ for all  $j \in \{1,\ldots,n-i\}\setminus
(k.l)$ where $ k, l$ are such that $k<l$  and
$d(\tilde{\mC}_{k}^{i-1},\tilde{\mC}_{l}^{i-1}) = 
\min_{1\leq m<q\leq  n-i+1}{d(\tilde{\mC_{m}}^{(i-1)},\tilde{\mC}_{q}^{(i-1)})}$.  Set
$\tilde{\mC}_k^{(i)} = \tilde{\mC}_k^{(i-1)} \cup
\tilde{\mC}_l^{(i-1)} $ and if $l<n-i+1$ then $\tilde{\mC}_l^{(i)} =
\tilde{\mC}_{n-i+1}^{(i-1)}$.  Set $i$ = $i+1$.  
The merging continues until the entire hierarchy has been built, or a
hierarchy with a pre-specified number of groups $K_\bullet$ have been obtained. 

\subsection{The $K$-means hierarchical ($K$-mH) cluster merging algorithm} \label{kmeans}
Our proposed algorithm removes scatter and then creates multiple
partitions, each formed by combining $K$-means and hierarchical
clustering. The algorithm has the following steps.
\begin{enumerate}
\item {\em Removing scatter from the dataset:} The algorithm first
  removes scatter from the dataset from consideration. 
\item {\em Finding a partition:} \label{kmHsteps} Our algorithm has
  two phases. The 
  first focuses on finding a (potentially) large number ($K_0$) of
  homogeneous spherical groups while the next merges these groups
  according to some criterion. We call these phases the $K$-means and
  hierarchical phases. The exact details of these phases are as follows:

\begin{enumerate}
  \item {\em The $K$-means phase:} For a given $K_0$ and
    initialization,  the $K$-means phase uses its namesake algorithm
    with multiple ($m$) initializations to identify $K_0$ homogeneous
    spherically-distributed groups. 
  This  phase yields $K_0$ groups $\{\mC_1,\mC_2,\ldots,\mC_{K_0}\}$ with means 
    $\bmu_1,\bmu_2,\ldots,\bmu_{K_0}$. 
Each obtained cluster $\mC_k$ is
  now considered to be one entity. Therefore, we now have $K_0$ entities labeled as
  ${\mC_1} , \mC_2, \ldots , {\mC_{K_0}}$ for consideration. 
\item \label{hcp} \emph{Hierarchical phase:} For given $K_*$ and distance
  $d(\cdot,\cdot)$, we successively merge the $K$-means groups as follows:
\begin{enumerate}
\item  Set $i^* = 1$ and $d^*_1 = 1$.  Define $\tilde{\mC_j}^{(1)} = \mC_j$ for all $j$.  
\item For $j \in 1... (K_0-i^*)$ $\tilde{\mC}_j^{(i^*+1)} = \tilde{\mC}_j^{(i^*)}$.  Find $ k, l$ such that $k<l$  and $d(\tilde{\mC}_{k}^{i^*},\tilde{\mC}_{l}^{i^*}) = \min_{1\leq m<q\leq (K_0-i^*+1)}{d(\tilde{\mC}_{m}^{i^*},\tilde{\mC}_{q}^{i^*})}$.  Set $\tilde{\mC}_k^{(i^*+1)} = \tilde{\mC}_k^{(i^*)} \cup \tilde{\mC}_l^{(i^*)} $ and if $l<K_0-i^*+1$ then $\tilde{\mC}_l^{(i^*+1)} = \tilde{\mC}_{K_0-i^*+1}^{(i^*)}$, define $d^*_{i^*} = d(\tilde{\mC}_{k}^{i^*},\tilde{\mC}_{l}^{i^*})$.  Set $i^*$ = $i^*+1$.  
\item If $i^* = K_0$ or $i^* = K_0-K_*+1$ terminate, else return to
  Step 2(b).  
\end{enumerate}
\end{enumerate}
\item {\em Forming multiple partitions and choosing the optimal
    $P_*$:}
\label{step3} Repeat Step~\ref{kmHsteps} $N=ML$ times with $M$ different
$K_0$s and $L$ different $K_*$s to form multiple partitions. Determine
the optimal hierarchical partition $P_*$.
\end{enumerate}
Our outlined algorithm   has several aspects that need
clarification. We do this next.
\subsubsection{Scatter Removal:}\label{scatter}  
Outliers or scatter can greatly influence clustering
performance~\citep{maitraandramler09}. Although many
methods~\citep{byersandraftery98,tsengandwong05,maitraandramler09}
 exist, we adopt the following straightforward approach to eliminating 
scatter. We use $K$-means with the largest of our candidate group
sizes ($G$) and multiple initializations ($K\sqrt{np}$) to obtain a $G$-means
partition. Observations in any of the $G$ groups that have less than
0.1\% of the size of the dataset are labeled as scatter and eliminated
from further consideration. This leaves us with $n^*$ observations
$\bX_1,\bX_2,\ldots,\bX_{n^*}$ (say) which we proceed with clustering
using $K-mH$.  
\subsubsection{Distance  between entities:}\label{dist2}
For the hierarchical phase of Step~\ref{kmHsteps}, we calculate the
distance  between two clusters obtained from the $K$-means step by
assuming (non-homogeneous) 
spherically-dispersed Gaussian-distributed groups in the
dataset. Specifically, we let    
$\bX_{1},\bX_{2}, \ldots, \bX_{n^*}$ be   independent $p$-variate observations
with $\bX_i\sim N_p(\bmu_{\zeta_i},\sigma^2_{\zeta_i}\bI) $, where $\zeta_i
\in \{1, 2, \ldots, K\}$ for $i=1, 2, \ldots, n^*$. Here we assume 
that $\bmu_k$'s are all distinct and that $n_k$ is the number of
observations in cluster $k$.  Then the density for the $\bX_i$'s is
given by  
$f(\bX) =  \sum_{k=1}^K{I(\bX \in
  {\mC}_k)\phi(\bX;\bmu_k,\sigma^2_k\text{I})},
\label{fx}
$
where $\mathcal C_k$ is a cluster indexed by the $N_p(\bmu_k,
\sigma^2_k\bI)$ density and $I(\bX \in {\mC}_k)$ is an indicator
function specifying whether observation $\bX$ belongs to the $kth$ group 
 having a $p$-dimensional multivariate
normal density $\phi(\bX;\bmu_k,\sigma^2_k\bI)\propto
\sigma_k^{-p}\exp\left[-\frac{1}{2\sigma_k^{2}}(\bX-\bmu_k)'
(\bX-\bmu_k)\right]$, $k =1, \ldots, K$. 
Define the distance measure
\begin{equation}
\label{dist}
\mathcal{D}_{k} (\bX_{i}) = \frac{(\bX_{i}-\bmu_k)'(\bX_{i}-\bmu_k)}{\sigma_k^{2}}
\end{equation}
and the variable

\begin{equation}\label{ylj}
Y^{j,l}(\bX) \; {=} \; \mathcal{D}_{j}({\bX})-\mathcal{D}_l({\bX}), \;
\textup{ where } {\bX} \in {\mC}_l,
\end{equation}
and $Y^{l,j}(\bX)$ similarly.  Using the spherically-dispersed
Gaussian models formulated above,  $Y^{j,l}(\bX)$ is a  random variable 
which  represents the difference in squared distances  of $\bX \in \mC_l$
to the center of ${\mC}_j$ and to the center of
${\mC}_l$. Then
$p^j_l = \Pr[Y^{j,l}(\bX)<0]$ is the probability that an observation
from $\mC_l$ is classified
into ${\mC}_j$ and is calculated as follows.
\begin{theorem}
  \label{theorem1}
Let $\bX \sim N_p(\bmu_l, \bSigma_l)$, with $\bSigma_l$ a
positive-definite matrix. Further, let $Y^{j,l}(\bX) = 
\mD_j(\bX) - \mD_l(\bX)$, where $\mD_k(\bX) = (\bX -
\bmu_k)'\bSigma_k^{-1}(\bX - \bmu_k)$ for $k\in\{j,l\}$. 
Let $\lambda_1, \lambda_2, \ldots,\lambda_p$ be the eigenvalues of
$\bSigma_{j|l} \equiv 
\bSigma_l^{\frac{1}{2}}\bSigma_j^{-1}\bSigma_l^{\frac{1}{2}}$ with
corresponding eigenvectors $\bgamma_1, \bgamma_2, 
...\bgamma_p$. Then $ Y^{j,l}(\bX)$ is distributed as 
%
  $\sum^p_{i=1}I(\lambda_i\neq 1)\left[(\lambda_i-1)U_i -
\lambda_i\delta_i^2/(\lambda_i-1)\right] 
+\sum^p_{i=1} I(\lambda_i= 1)\delta_i(2Z_i+\delta_i), $ 
where $U_i's$ are independent non-central $\chi^2$ random
variables with one degree of freedom and non-centrality parameter
$\lambda_i^2\delta_i^2/(\lambda_i-1)^2$ with $\delta_i =
\bgamma_i'\bSigma_l^{-\frac{1}{2}}(\bmu_l-\bmu_j)$ for $i \in
\{1,2,...,p\}\cap\{i:\lambda_l \neq 1 \}$, independent of $Z_i$'s,
which are independent standard normal random variables, for $i \in
\{1,2,...,p\}\cap\{i:\lambda_i = 1 \}.$\\
\end{theorem}
\begin{proof}
Let $\bxi\sim N_p(0,{\bI})$. Since $\bX \stackrel{d}{=}
\bSigma_l^{\frac{1}{2}}\bxi +\bmu_l$, 
we have
\begin{eqnarray} \label{eq13}
Y^{j,l}(\bX)&{=}&
{\bX }'(\bSigma_j^{-1}-\bSigma_l^{-1}){\bX } +
2{\bX}'(\bSigma_l^{-1}\bmu_l-\bSigma_j^{-1}\bmu_j)                                                                  +
\bmu_j'\bSigma_j^{-1}\bmu_j-\bmu_l'\bSigma_l^{-1}\bmu_l\nonumber\\
&\stackrel{d}{=}& (\bSigma_l^{\frac{1}{2}}\bxi
+\bmu_l)'(\bSigma_j^{-1}-\bSigma_l^{-1})(\bSigma_l^{\frac{1}{2}}\bxi
+\bmu_l) + 2(\bSigma_l^{\frac{1}{2}}\bxi
+\bmu_l)'(\bSigma_l^{-1}\bmu_l-\bSigma_j^{-1}\bmu_j) +
\bmu_j'\bSigma_j^{-1}\bmu_j-\bmu_l'\bSigma_l^{-1}\bmu_l\nonumber\\{}&=&
\bxi'(\bSigma_{j|l}-{\bI})\bxi +
2\bxi'\bSigma_l^{\frac{1}{2}}\bSigma_j^{-1}(\bmu_l-\bmu_j) +
(\bmu_l-\bmu_j)'(\bSigma_j^{-1})(\bmu_l-\bmu_j)
\end{eqnarray}
where
$\bSigma_{j|l}=\bSigma_l^{\frac{1}{2}}\bSigma_j^{-1}\bSigma_l^{\frac{1}{2}}$.
Let the spectral decomposition of $\bSigma_{j|l}$ be given by
$\bSigma_{j|l} = \bG_{j|l}\bL_{j|l}\bG'_{j|l}$,
 where $\bL_{j|l}$ is a
diagonal matrix containing the eigenvalues $\lambda_1, \lambda_2,
\ldots \lambda_p$ of $\bSigma_{j|l}$, and $\bG_{j|l}$ is an
orthogonal matrix containing the eigenvectors $\bgamma_1, \bgamma_2,
\ldots,\bgamma_p$ of $\bSigma_{j|l}$. Since $\mathbf{Z} \equiv{\bG_{j|l}}'\bxi
\sim  N_p(0,{\bI}) $ as well, we get from (\ref{eq13}) that
\begin{eqnarray}
Y^{j,l}(\bX)&\stackrel{d}{=}&
\bxi'(\bG_{j|l}\bL_{j|l}\bG'_{j|l}-\bG_{j|l}\bG'_{j|l})\bxi 
 +2\bxi'(\bG_{j|l}\bL_{j|l}\bG'_{j|l}\bSigma^{-\frac{1}{2}}_1)(\bmu_l-\bmu_j)
                                                                           +(\bmu_l-\bmu_j)'(\bSigma^{-\frac{1}{2}}_1\bG_{j|l}\bL_{j|l}\bG'_{j|l}
\bSigma^{-\frac{1}{2}}_l)(\bmu_l-\bmu_j)\nonumber\\{}&=&
(\bG'_{j|l}\bxi)'(\bL_{j|l}-{\bI})(\bG'_{j|l}\bxi)+2(\bG'_{j|l}\bxi)'(\bL_{j|l}\bG_{j|l}\bSigma^{-\frac{1}{2}}_l)(\bmu_l-\bmu_j)+(\bmu_l-\bmu_j)'(\bSigma^{-\frac{1}{2}}_l\bG_{j|l}\bL_{j|l}
\bG_{j|l}'\bSigma^{-\frac{1}{2}}_l)(\bmu_l-\bmu_j)\nonumber\\&\stackrel{d}{=}&
\sum_{i=1}^p\left[ (\lambda_i-1){Z_i}^2+2\lambda_i\delta_i
 Z_i+\lambda_i\delta_i^2\right], \label{eq14.5}
\end{eqnarray}
where $\delta_i, i=1,2,\ldots, p$ are as in the statement of the
theorem. We can simplify (\ref{eq14.5}) further based on the values of
$\lambda_i$: If $\lambda_i>1$:
$(\lambda_i-1)Z_i^2 +2\lambda_i\delta_i
Z_i+\lambda_i\delta_i^2=(\sqrt{\lambda_i-1}
Z_i+\lambda_i\delta_i/\sqrt{\lambda_i-1})^2-\lambda_i\delta_i^2/
(\lambda_i-1)$, while 
for $\lambda_i<1$: $(\lambda_i-1){Z_i}^2 
+2\lambda_i\delta_i Z_i+\lambda_i\delta_i^2=-(\sqrt{1-\lambda_i} 
Z_i-\lambda_i\delta_i/\sqrt{1-\lambda_i})^2-\lambda_i\delta_i^2/ 
(\lambda_i-1)$. In both cases, $(\lambda_i-1){Z_i}^2+2\lambda_i\delta_i
 Z_i+\lambda_i\delta_i^2$ is distributed as a 
$(\lambda_i-1)\chi^2_{l,\lambda^2_i\delta^2_i/(\lambda_i-1)^2}$-random
variable shifted by $-\lambda_i\delta_i^2/(\lambda_i-1)$. When
$\lambda_i = 1$, $(\lambda_i-1)  
Z_i^2 +2\lambda_i\delta_i Z_i+\lambda_i\delta_i^2 =
2\delta_iZ_i+\delta_i^2.$ The theorem follows from some further minor rearrangement of terms.
\end{proof}
\begin{corollary}
\label{cor2}
Let $\bX\sim N_p(\bmu_l,\sigma_l^2\bI)$. Define $\mD_k(\bX)$ 
 as in \eqref{dist}. If $\sigma_l=\sigma_j$, we have 
$Y^{j,l}(\bX) \sim N(\norm{\bmu_j-\bmu_l}^2/\sigma_l^2,4\norm{\bmu_j-\bmu_l}^2/\sigma_l^2)$,
otherwise $Y^{j,l}(\bX) \sim (\sigma_l^2/\sigma_j^2 -1)
\chi^2_{p;\norm{\bmu_l-\bmu_j}^2/{(\sigma_j^2-\sigma_l^2)}^2 } - \norm{\bmu_l-\bmu_j}^2/(\sigma_l^2-\sigma_j^2)$.
\end{corollary}
\begin{proof}
  Here,   $\lambda_i \equiv 
  \sigma_l^2/\sigma_j^2$, $\bgamma_i$ is the 
  $i$th unit vector, and 
  $\sum_{i=1}^p\delta_i^2=\norm{\mu_l-\mu_j}^2/\sigma_l^2$. Also,
  the sum of $p$ 
  independent $\chi^2_{1;\tau_i^2}$ random variables has the same
  distribution as a  $\chi^2_{p;\sum_{i=1}^p\tau_i^2}$ random
  variable.  The proof follows from Theorem~\ref{theorem1}.
\end{proof}
Corollary~\ref{cor2} provides an easy calculation for $p^j_l$ and
$p_j^l$. Note, however, that for large $\delta$ (and/or $p$) the
$\chi^2_{p;\delta}$ cumulative distribution function is not evaluated
accurately so we approximate this quantity by the corresponding
cumulative distribution function of the $N(p+\delta, 2(p+2\delta))$
random variable~\cite[for details, see][pages 22-24 and problem
1.8]{muirhead05}. The net result is that we have approximate but very
speedy and accurate calculations. 
This is important because our hierarchical phase uses  
the distance measure between groups $\mC_j$ and $\mC_l$ that we define
to be 
\begin{equation}\label{distmeasure}
d(\mC_j,\mC_l) = 1-(p^j_l+p^l_j)/2.
\end{equation}
We now adapt this distance measure to the initial and iterative parts
of the hierarchical phase. At the beginning of the hierarchical phase
(equivalently, the conclusion of the $K$-means phase), we have  $K_0$
entities with labels 
$\mC_{1},\ldots, \mC_{K_0}$.  For $1 \leq k\leq K_0$, we already have
the $\hat\bmu_k$s while the covariance matrix 
($\hat\sigma^2_k\bI$) is estimated by setting $\sigma_k^2$ as the trace of
the variance-covariance matrix of $\mC_{k}$ scaled by $p$. For
subsequent stages, \eqref{distmeasure} is 
updated by replacing the distance between an entity (say, $\mC_l$) and
a merged entity (say, $\mC_j \cup \mC_k$) as $d(\mC_l, \mC_j\cup\mC_k)
= \min \{d(\mC_l, \mC_j), d(\mC_l,\mC_k)\}$. A convenient aspect of this
strategy is that off-the-shelf hierarchical clustering software (for
example, the {\tt hclust} function in R) with single linkage can be
used to implement the hierarchical phases of our $K$-mH algorithm.

\subsubsection{Forming $N$ partitions and choosing $P_*$:}\label{P*}
Step~\ref{kmHsteps} of the $K-mH$ algorithm produces one
partition starting with $K_0$ entities ending with $K_*$
clusters. Step~\ref{step3} runs Step~\ref{kmHsteps}  $N=ML$ times, where
$M$ is the number of $K_0$s and $L$ is the number of $K_*$s used. We
discuss choosing $K_0$ and $K_*$ next.
 \paragraph{Choosing candidate $K_0$:}\label{chooseK0}
Our proposal for $K_0$ involves chooses a range of values
$\{k_1,k_2,\ldots,k_m\}$, $m\geq M$ for which we calculate $C_{k_1},
C_{k_2},\ldots,C_{k_m}$ using \citet{Krzanowski88}'s suggestions of
Section \ref{diffg}.  We sort these values to get $C_{g_1}\geq
C_{g_2}\geq,\ldots,\geq C_{g_m}$, where the set $\{g_1, g_2 ,\ldots,
g_m\} = \{k_1,k_2,\ldots,k_m\}$. However, instead of setting
$K_0\equiv g_1$ as 
recommended by \citet{Krzanowski88}, we propose running 
Step~\ref{kmHsteps} of our algorithm for each $K_0\equiv K_o^{(i)}$,
where $K_0^{(1)} = g_1, K_0^{(2)} = g_2,
\ldots , K_0^{(M)} = g_M$, that is, for the numbers of clusters
corresponding to the  $M$ highest $C_{g_j}$s. So we run the $K$-means
phase $M$ times with $K_0 = K_0^{(i)}$ for $i=1,2,\ldots,M$, with 
$K_0^{(1)} = g_1, K_0^{(2)} = g_2,  \ldots , K_0^{(M)} = g_M.$ For
each of these runs, we set $K_*\equiv K_*^{(i)}$ in the hierarchical
phase and in the manner described next.
\paragraph{Choosing candidate $K_*$:}\label{chooseK}
For each value of $K_0^{(i)}$, we use $K_*$ if the number of desired
general-shaped clusters is known and then we set 
$L=1$. When $K_*$ is unknown, we
obtain a range of $K_*$s by defining change-points ($CP$s) as 
$CP_{k} = d^*_{k+1} - d^*_{k}$ (for $k = 1,\ldots, K_0^{(i)}$)
where  $d^*_1 \leq d^*_2 \leq \ldots\leq
d^*_{K_0^{(i)}}$ are  calculated during Step~\ref{hcp} of the
algorithm.   We sort these $CP$-values to get $CP_{q_1}\geq
CP_{q_2}\geq,\ldots,\geq CP_{q_{K_0^{(i)}-1}}$, where the set $\{q_1, q_2 ,\ldots,
q_{K_0^{(i)}-1}\}$ is some appropriate permutation of the set 
$\{2,3,\ldots,K_o^{(i)}\}$. We consider the first $L$ of these
values. That is, we define $k_{i,1} = q_1, k_{i,2} =
q_2,\ldots, k_{i,L} = q_L$ as in Section \ref{chooseK} for when we
have $K_0 = K_0^{(i)}$.  Then for each $K_0^{(i)}$ we obtain  $L$
partitions using $K_* = k_{i,j}$ for $j \in \{1,2, \ldots, L\}.$ Thus,
we arrive at $N=ML$ partitions $\{P_1,P_2,\ldots,P_N\}$ for all
combinations of $K_0$ and $K_*$. 

\subsubsection{Visualizing partitions and choosing optimal
  $K_*$:}\label{simmatrix}  
We extend \citet{fred05}'s ideas to visualize the stability
and variability in our partitions. Consider the $n\times n$ similarity
matrix $\Psi$ with $(i,j)$th entry $\psi_{ij} = n_{i,j}/N$, where $n_{ij}$ is the
number of times 
that the $i$th and $j$th observations are in the same cluster across
the $N$ partitions obtained from Section~\ref{P*}. 
We display $\Psi$ via a clustered heatmap. The heatmap provides
indication into both the structure and stability of the clustering. We
can use this heatmap to decide on $K_*$ by determining all partitions
which remain after thresholding below $\psi_{ij}=0.5$. We use two
alternative choices in forming these partitions. In the first
case, if the off-diagonal $\psi_{ij}$s are generally small or
uncertain ({\em i.e.} their mean is small or their coefficient of
variation is high), we use 
single-linkage otherwise we use complete linkage. As
 with \citet{fred05}, heatmaps create very large files for large $n$
 so we then use a 
random sample of the observations. We replicate 
this process $B$ times to assess the variability in $K_*$. 
\subparagraph{Final partition:}
With $K_*$ known or determined through the methods
of Section~\ref{simmatrix}, we have $L =
1$ as per Section~\ref{P*}. Then, with the 
$N=M$ partitions, we pick the clustering that is most similar to the
other $N-1$ partitions. This is operationally implemented by defining
the  $N\times N$ matrix $\mathcal{W}$ where 
$\mathcal{W}_{i,j} = \mathcal{R}_{i,j}$, where $\mathcal{R}_{i,j}$ is
the value for the Adjusted Rand Index~\citep{hubertandarabie85}
between partitions $P_i$ and $P_j$.  Define the objective function: 
$\bar{\mathcal{W}_i} = \sum_j{\mathcal{W}_{i,j}/N}. $
Then, we choose $P_*$ to be the partition that best matches $\Psi$ in
the sense of maximizing the objective function.  Thus,
$
 P_*= \{P^i: \bar{\mathcal{W}_i} = \max\limits_{1<j<N}{\bar{\mathcal{W}_{j}}}\}
$
is our choice for the final clustering and represents the
partition that is most similar to all the other candidate partitions.

In this section, we have developed an algorithm that combines elements
of $K$-means and hierarchical clustering to identify general-shaped
clusters. All steps in our algorithm are easily implemented using
existing software libraries and functions in R~\citep{R} and other
programming languages. We next evaluate  performance of our algorithm
on several datasets.

\section{Performance Evaluations}\label{examples}
We now evaluate $K$-mH on 
simulated and real datasets to highlight the
strengths and weaknesses of our methodology. 
We compare $K$-mH to the EAC (FJ) of \citet{fred05} (FJ), cluster
merging (CM) of \citet{baudryetal08}, generalized single linkage with
nearest-neighbor density estimate (GSL-NN) \citep{nugent10},
DEMP~\citep{hennig10} and DEMP+~\citep{melnykov16}.  We used 
R \citep{R} for all methods except for CM which used 
Matlab code provided in the supplemental material of
\citet{baudryetal08}. For CM, we used the ``elbow rule'' on the plot
of entropy variation against $K$ to determine
$K$ \citep{baudryetal08} while for GSL-NN, we used the procedure
in Section 7 of \citet{nugent10}. For FJ, er used the method in
Section 3.3 of \cite{fred05}.  
Our $K$-mH algorithm used $M=\mbox{min}\{10,\lfloor \sqrt{np}/10\rfloor\}$
(where $\lfloor x\rfloor$ is
the smallest integer less than or equal to $x$), $L=3$ (before
estimating $K_*$), 
$B=100$ and $G=\lfloor\sqrt{n}\rfloor$. 
In all cases, we used  
$\mathcal{R}$ \citep{hubertandarabie85} calculated 
between the true and estimated partitions to quanitify performance.

\subsection{Two-dimensional Examples}\label{2d}
We first illustrate and evaluate performance on many two-dimensional
examples found in the literature. 

\subsubsection{Smaller-sized Datasets:} 
\begin{figure}[h!t]
\centering
\mbox{
\subfloat[Banana-clump Dataset]{
\includegraphics[angle=0,totalheight=2.25in]{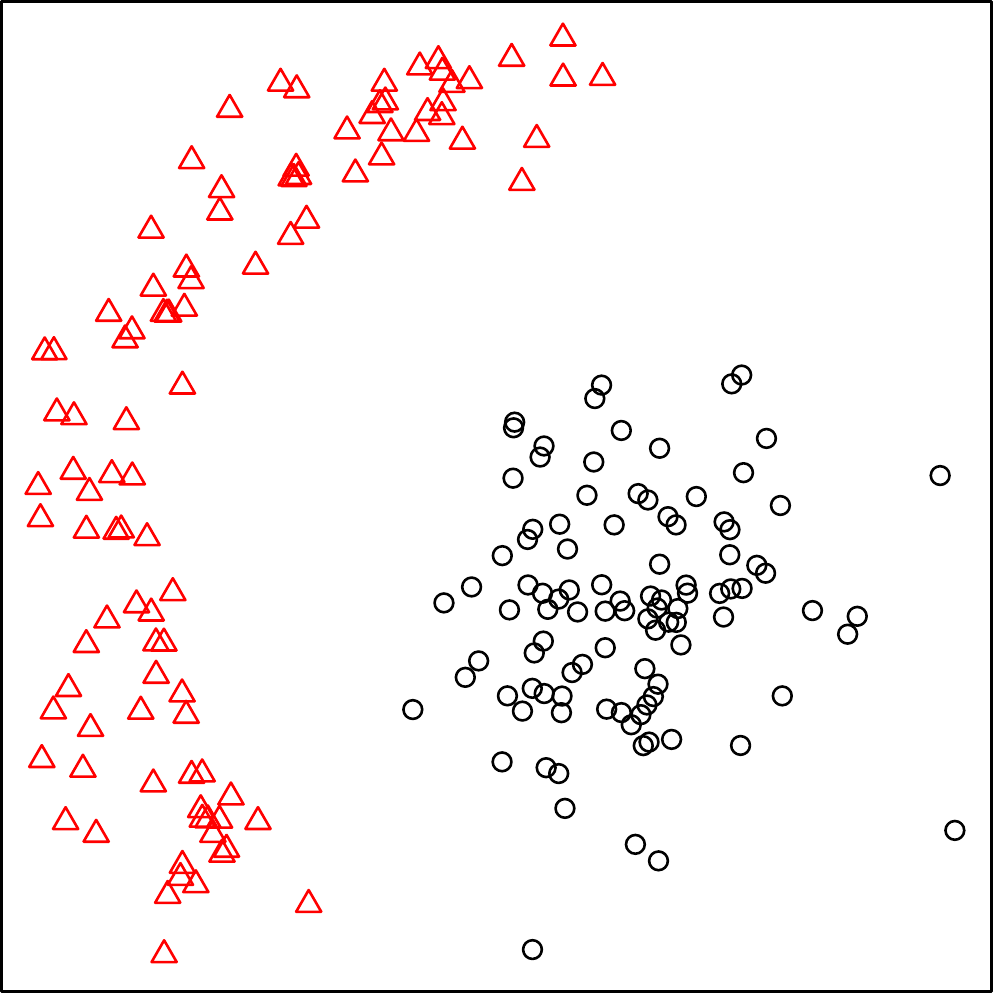}}
\hspace{0.1in}
\subfloat[Heatmap of Partitioning]{
\includegraphics[angle=0,totalheight=2.25in]{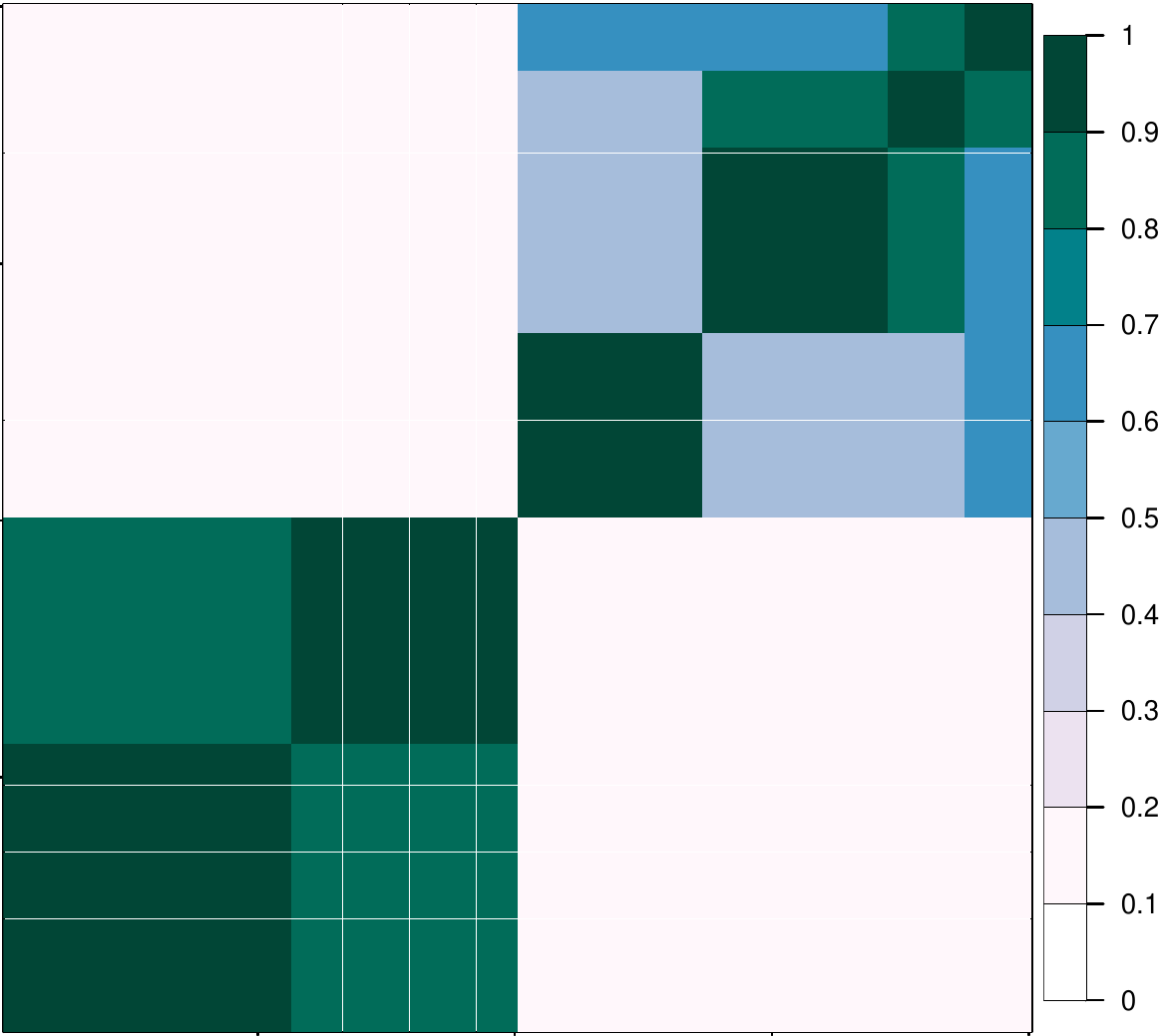}}
}
\caption{(a) $K$-mH partitioning of the Banana-clump dataset 
  and (b) heatmap illustrating clustering uncertainty and stability.}
\label{banclump}
\end{figure}
The Banana-clump dataset~(Figure \ref{banclump}a) of \citet{nugent10} has
200 observations. FJ, DEMP+, GSL-NN and $K$-mH all reproduce the
original partitioning but DEMP and the ``elbow'' approach of CM
suggest three groups with the banana essentially halved. Figure
\ref{banclump}b  
displays the heatmap obtained as part of $K$-mH. 
Two large clustered blocks are indicated with uncertainty over whether
the upper right block should be partitioned further. (It is this
partitioning that DEMP and CM go for.) Therefore, the heatmap displays
the uncertainty and structure in the partitioning, but the $K$-mH algorithm
chooses two groups. 
\begin{figure}[h!t]
\centering
\vspace{-0.12in}
\hspace{-0.1in}
\mbox{
\subfloat[Bullseye Dataset]{
\includegraphics[angle=0,totalheight=2.5in]{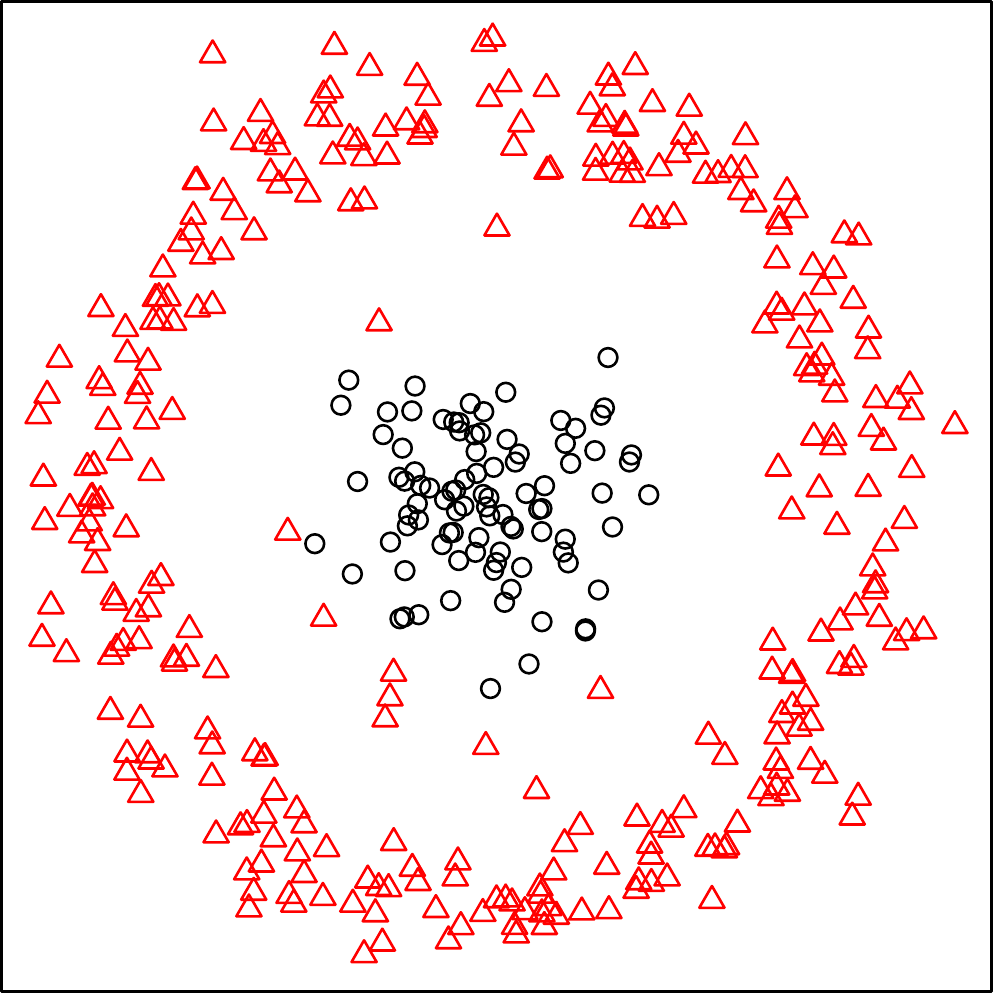}}
\subfloat[Heatmap of Partitioning]{
\includegraphics[angle=0,totalheight=2.5in]{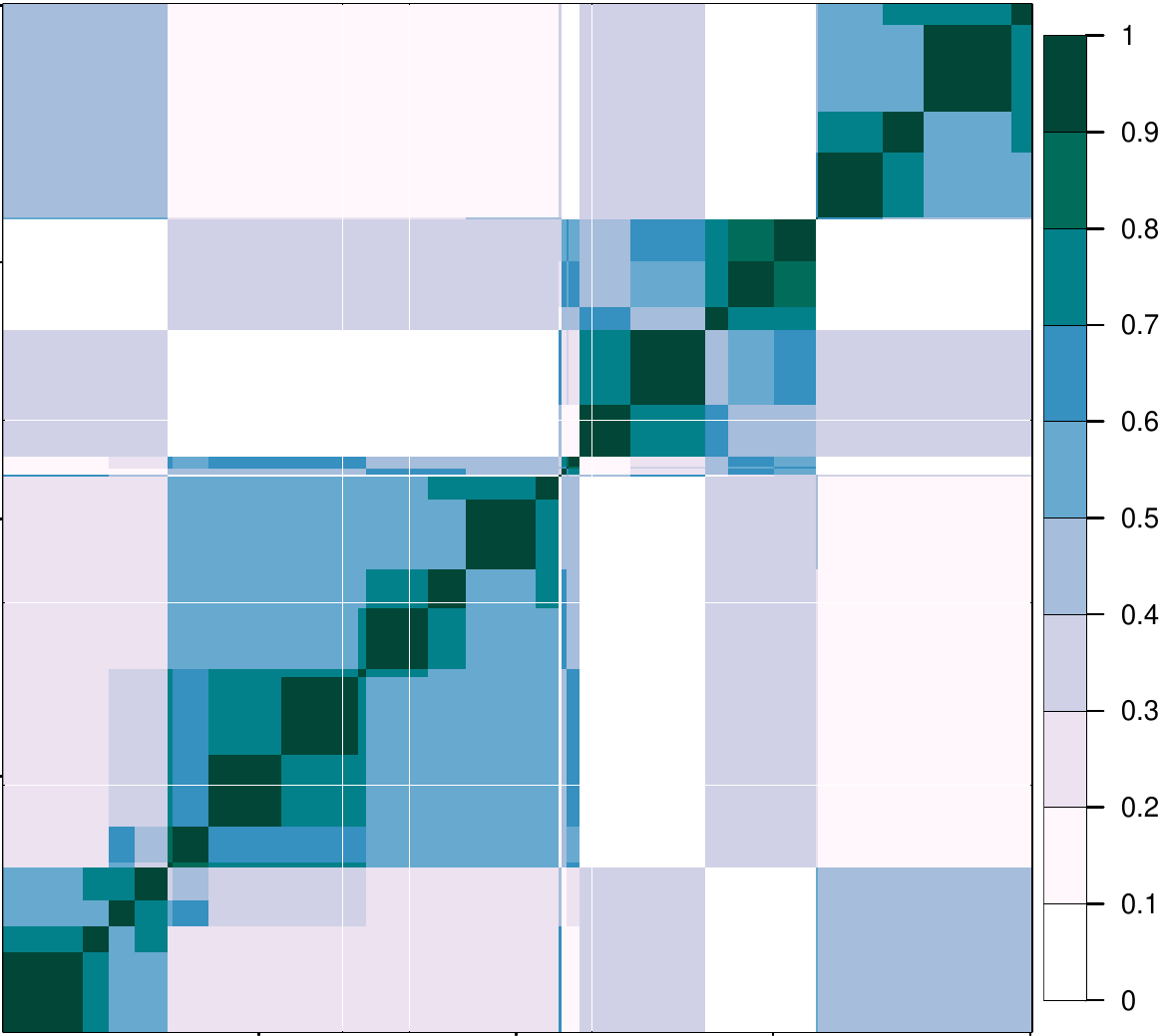}}}
\caption{(a) $K$-mH partitioning of the Banana-clump dataset 
  and (b) heatmap illustrating clustering uncertainty and stability.}
\label{bulleye}
\end{figure}

Revisiting the Bullseye dataset of Figure~\ref{bull}, we find that 
FJ, GSL-NN and $K$-mH produce good
partitions~(Figure~\ref{bulleye}a) with
$\mathcal{R}\geq 0.99$ but DEMP, DEMP+ and CM perform poorly with the
outer ring broken into several further groups. The
heatmap~(Figure~\ref{bulleye}b) indicates a lot of uncertainty but
the methodology of Section~\ref{simmatrix} suggests two groups. 
\subsubsection{The Banana-spheres dataset:} This 
dataset has two separated banana-shaped half rings of 250
observations each that are surrounded by a third group in the shape of
a full ring of 1500 observations. 
 The observations in each group were simulated using
 pseudo-random realizations from different bivariate normal
 distributions with means that followed the central path of each
 shape. An additional 15 outlying observations from each cluster were added to
 provide a dataset of 3015 observations. 
  \begin{figure}[!ht]
\centering
\hspace{-0.15in}
\mbox{
\subfloat[$\mathcal{R} = 0.99$, $K$ = 3]{
\includegraphics[angle=0,width=1.6in]{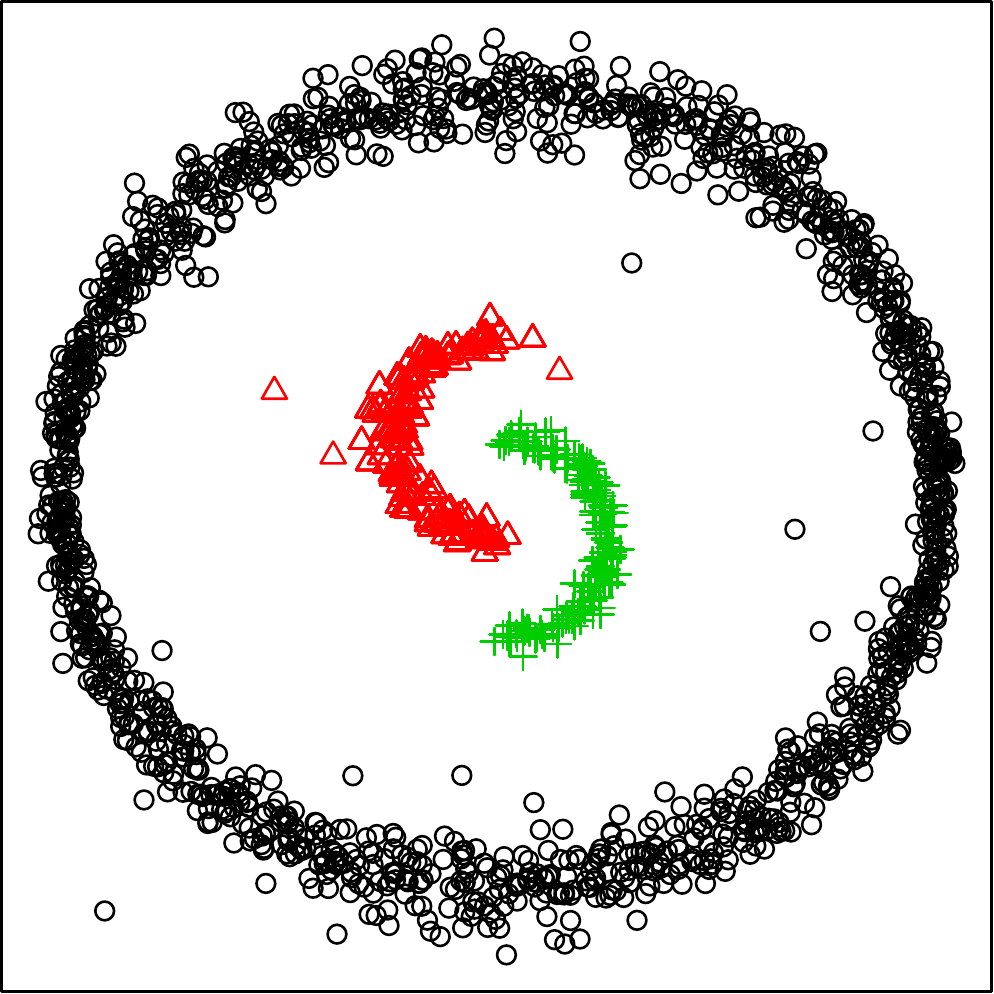}}
\subfloat[$\mathcal{R} = 0.95$, $K$ = 5]{
\includegraphics[angle=0,width=1.6in]{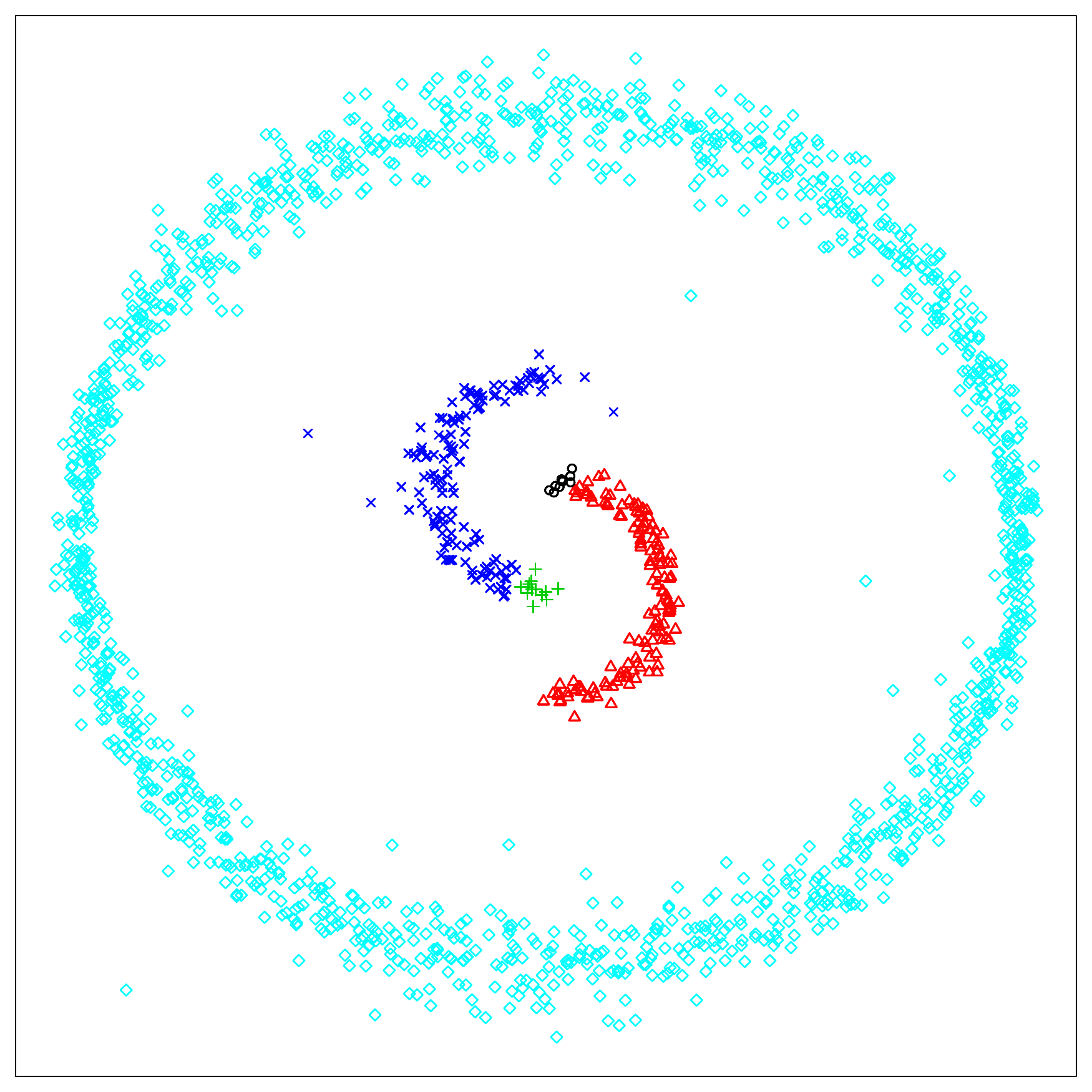}}
\subfloat[$\mathcal{R} = 0.74$, $K$ = 2]{
\includegraphics[angle=0,width=1.6in]{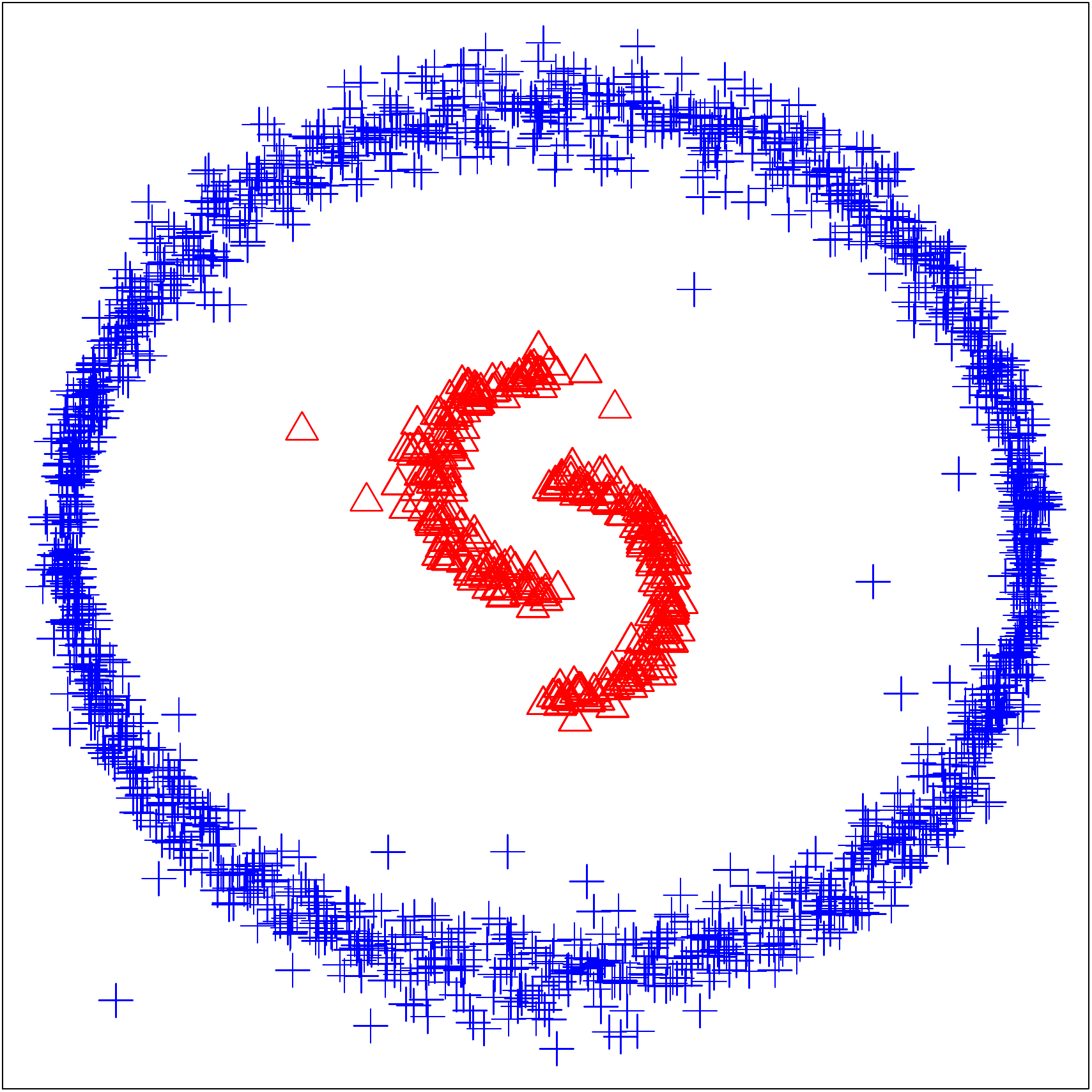}}
\subfloat[$K$-mH heatmap]{
\includegraphics[angle=0,width=1.8in]{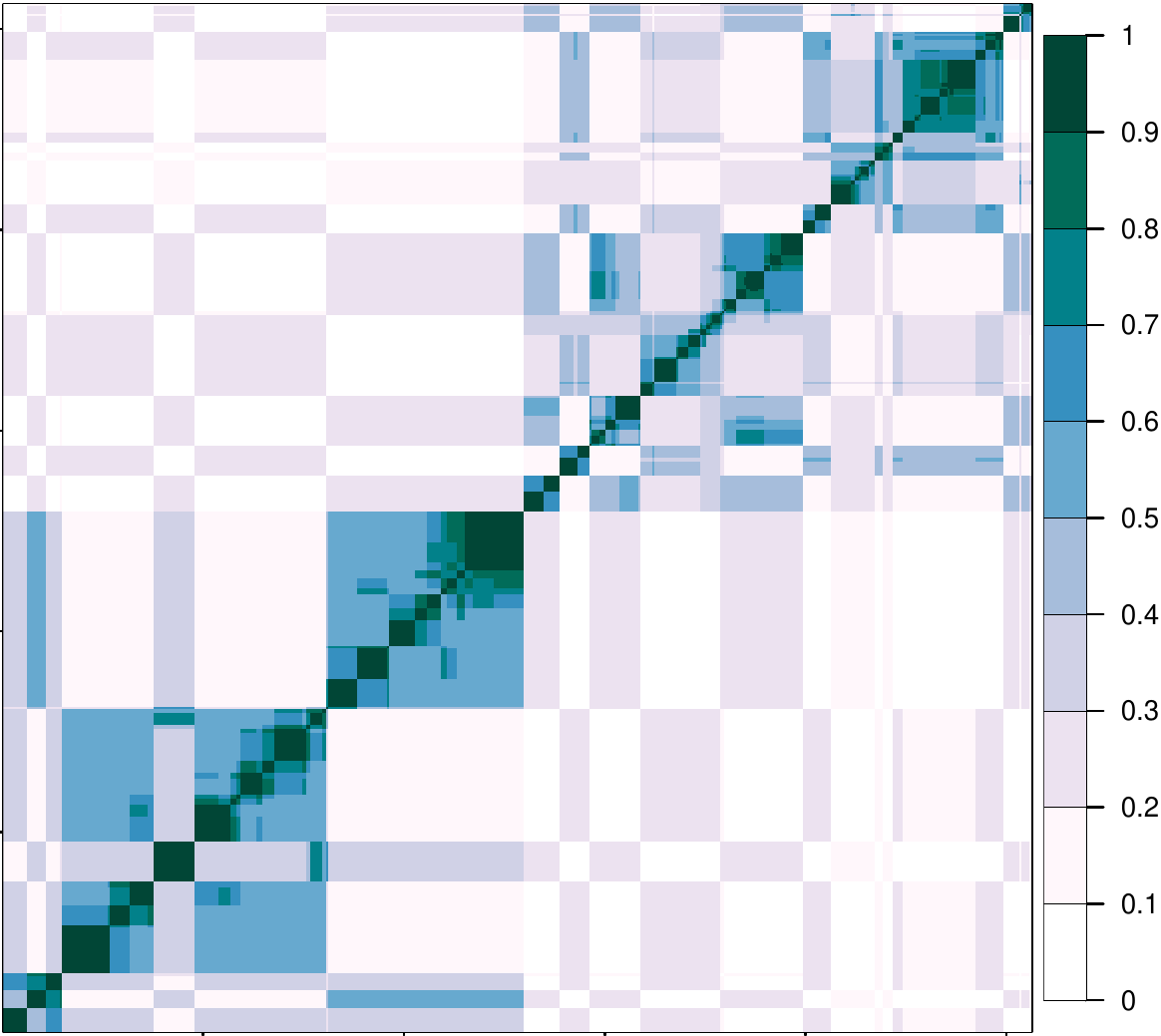}}
}
\caption{Top three partitionings of the Bananas-sphere dataset using
  (a) $K$-mH (b) FJ and (c) GSL-NN. Captions indicate
  estimated number of groups and $\mathcal{R}$ between estimated and true
  groupings. 
  (d) $K$-mH heatmap for stability of groupings.
}
\label{banana}
\end{figure}
Figures \ref{banana}a-c display the top three performers. 
$K$-mH chooses three groups with $\mR=0.99$ while
FJ chooses a 5-groups partition: however,  the partitioning is still quite
good ($\mathcal{R}= 0.95$).  GSL-NN
suggests 2 clusters while the elbow plot of CM  provides $K$ = 11 and $\mR =
0.53$.  Both DEMP ($\mR= 0.29$)and DEMP+ ($\mR=0.45$) do worse.
Further the heatmap (Figure~\ref{banana}d) shows the structure in the
dataset. While there are between 2 and 3 clear groups, there is also
indication of the complicated structure of each group as well as the
outliers.

\subsubsection{The SCX Dataset:}
This dataset has a variety of cluster shapes and sizes, with 
three separated C-shaped groups rotated at different angles,
a large S-shaped group and four small X-shaped groups.  Twenty
outlying observations are added to the clusters for a
total of 3420 observations. Here, $K$-mH partitioning~(Figure~\ref{Scross}) is
near-perfect (with two observations misclassified as scatter and not
displayed in the dataset) while FJ is the 
next best performer. CM, DEMP and DEMP+ perform similarly, but GSL-NN
finds 7 groups~($\mR = 0.53$) clusters, with the S and 4 crosses all
placed in one group and the two lower C's split into 2 and three
groups, respectively. The heatmap indicates uncertainty with
4 large groups with further definition and $K_*$  not easily
identified. This uncertainty is 
reflected in the estimated $K_*$s which were 7,  8,  9,
and 10, with frequency of occurrence 28, 48, 22, and 2\% of
the time, respectively. The median estimated $K_*=8$  
yields the perfect solution of Figure~\ref{Scross}a. 
\begin{figure}[h!t]
\centering
\mbox{\subfloat[$\mathcal{R} = 0.99$, $K$ = 8]{
\includegraphics[angle=0,width=1.6in]{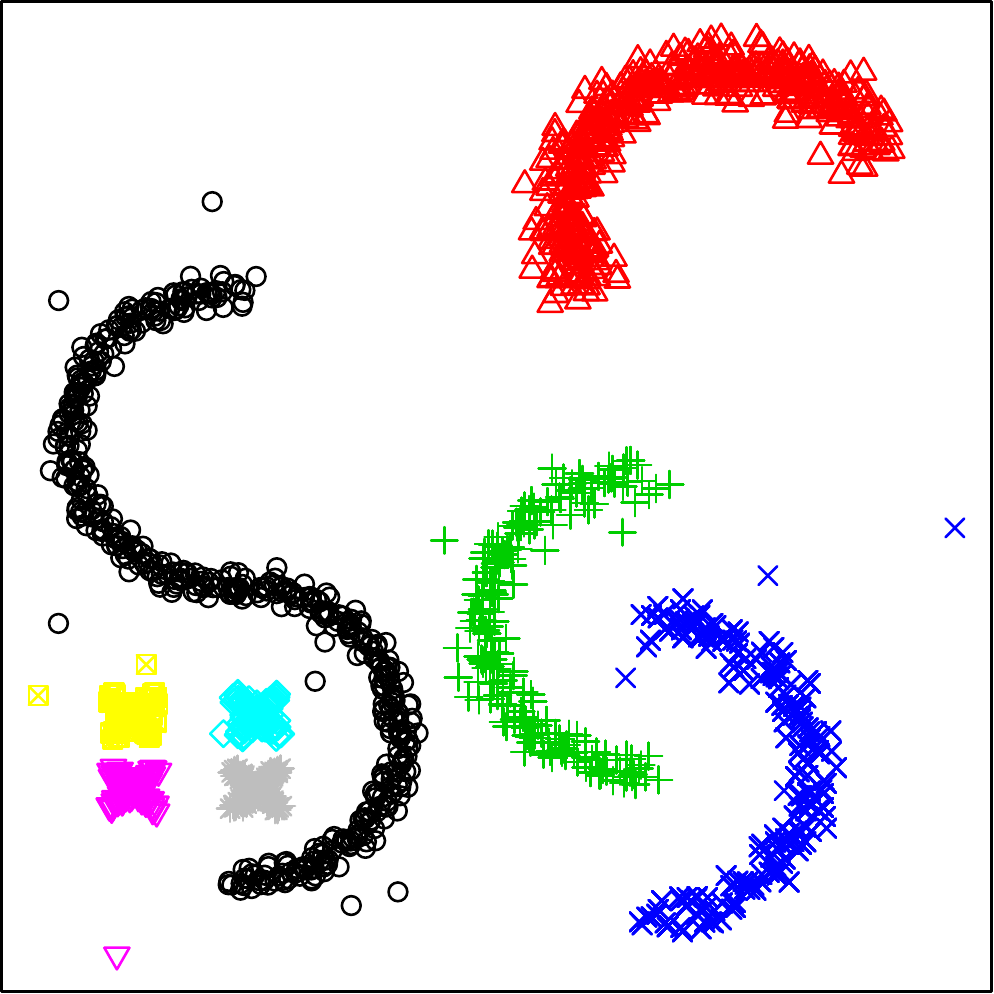}}
\subfloat[$\mathcal{R} = 0.89$, $K$ = 8]{
\includegraphics[angle=0,width=1.6in]{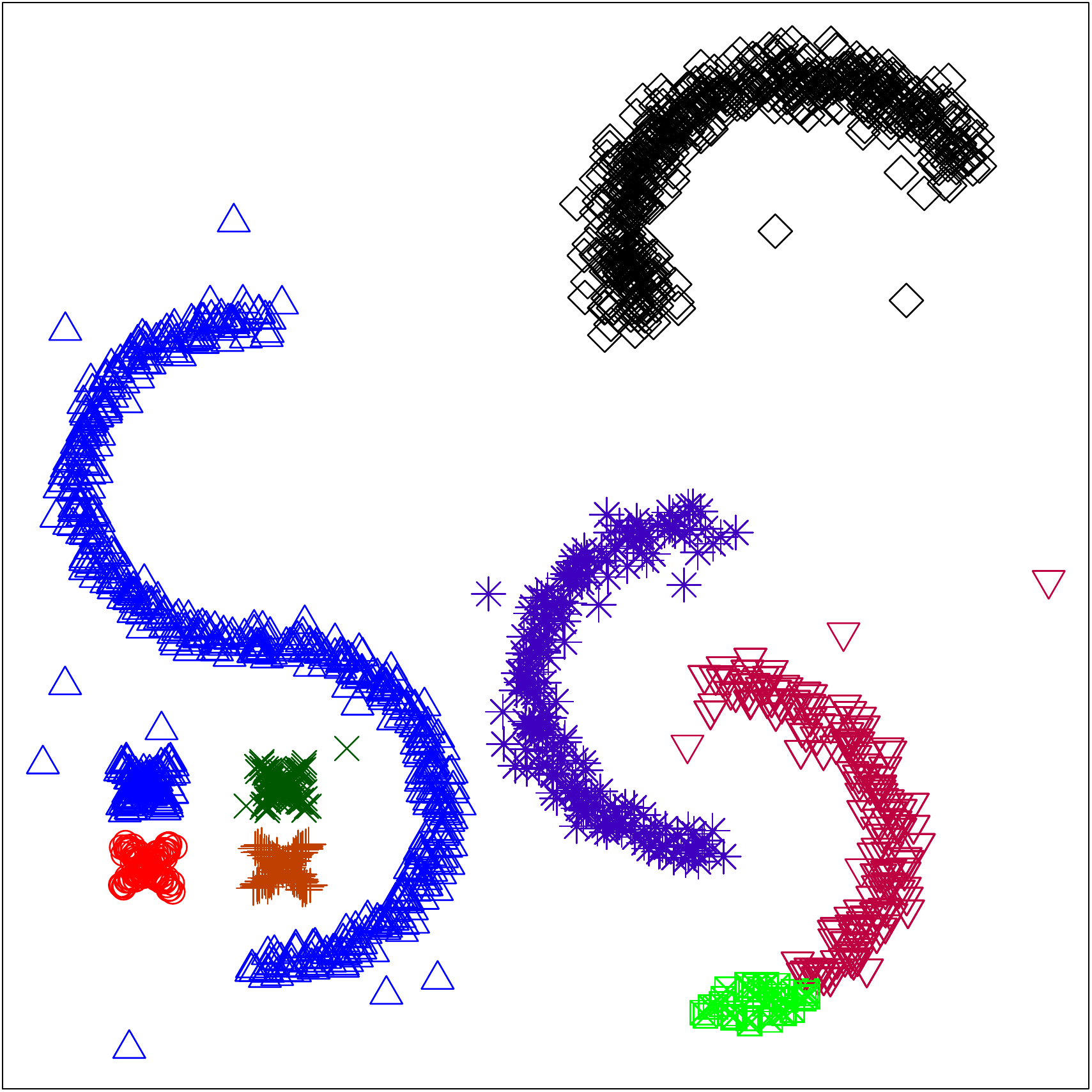}}
\subfloat[$\mathcal{R} = 0.78$, $K$ = 9]{
\includegraphics[angle=0,width=1.6in]{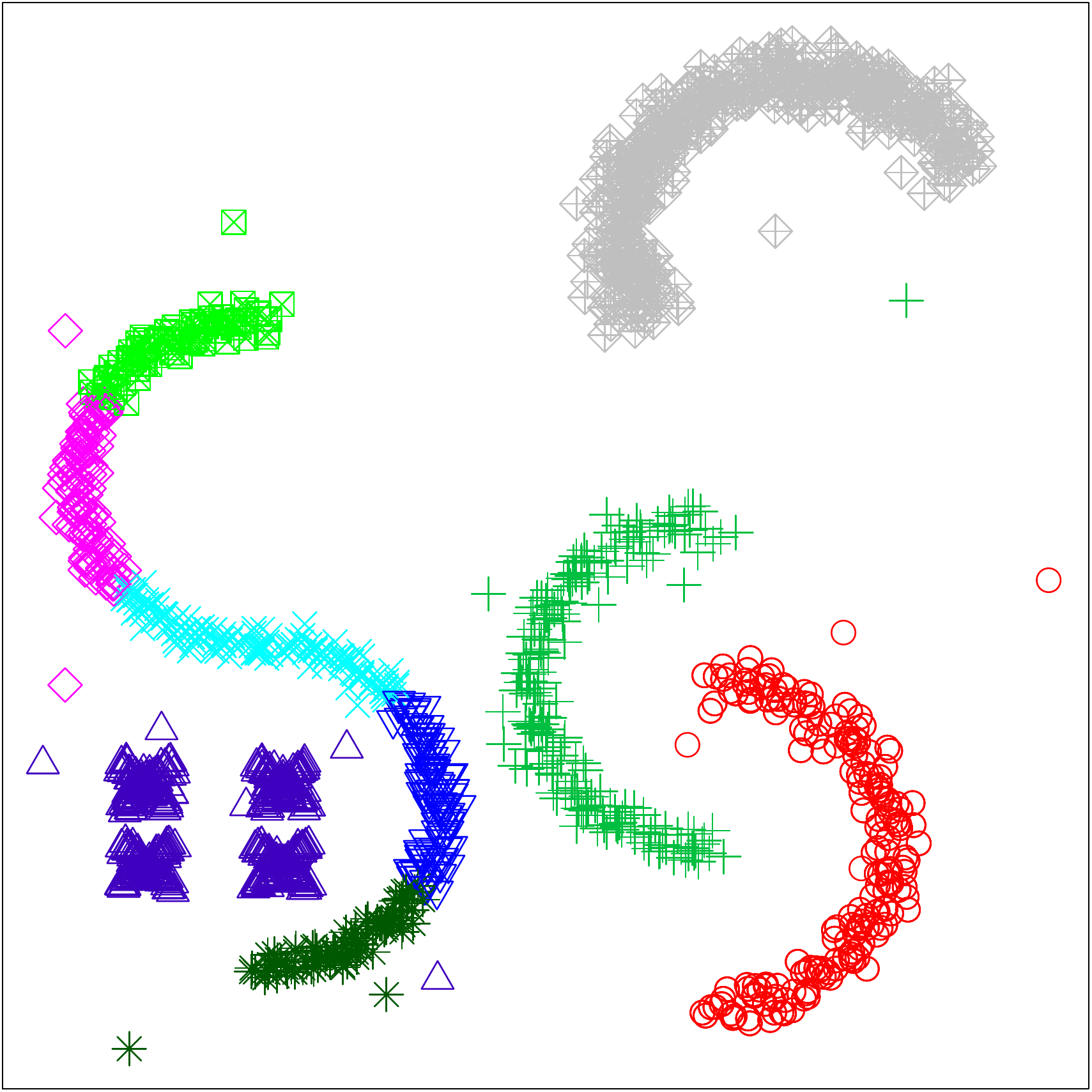}}
\subfloat[$K$-mH heatmap]{
\includegraphics[angle=0,width=1.8in]{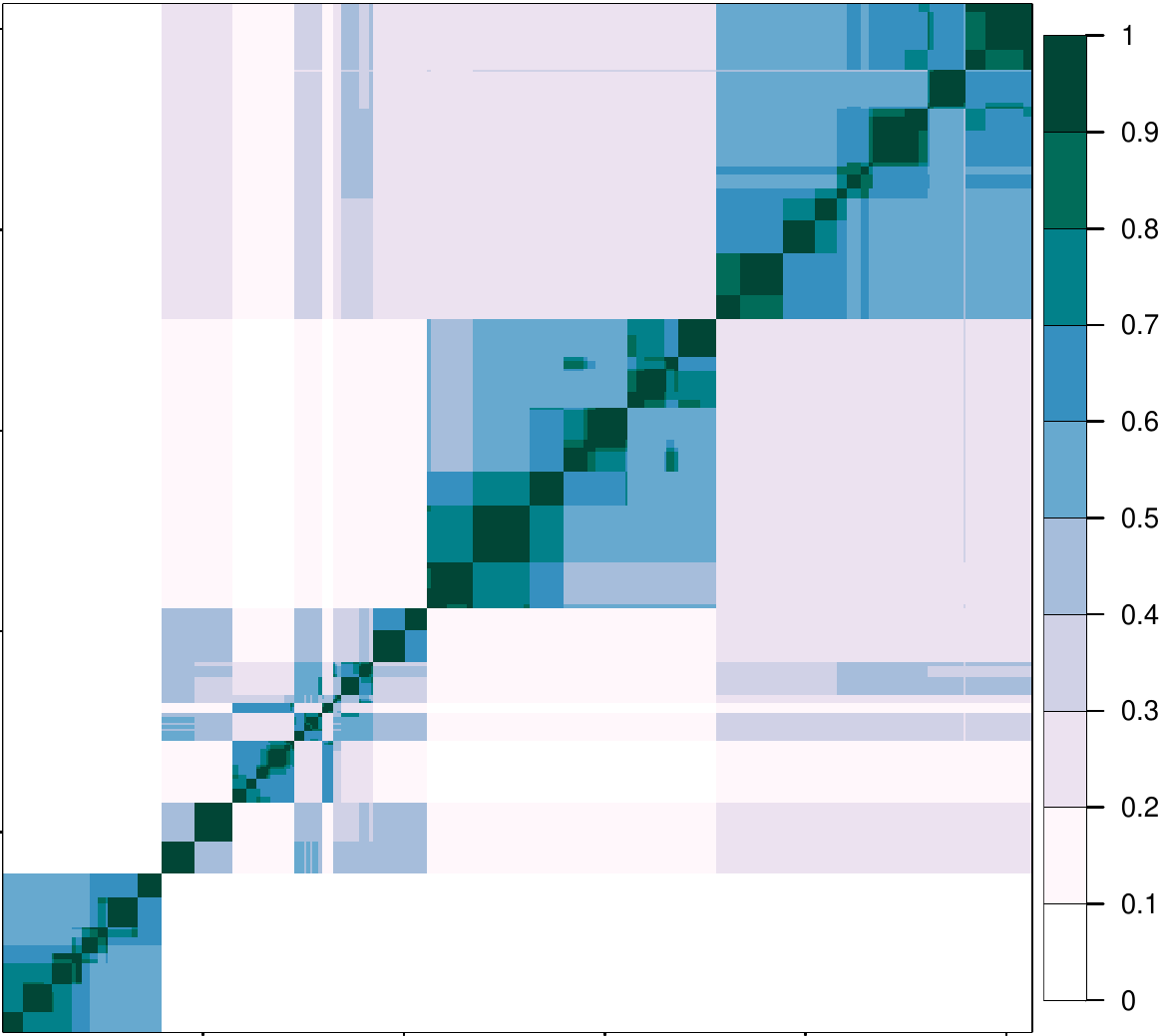}}}
\caption{Top three performers for SCX: (a) $K$-mH (b) FJ and (c) CM
  and (d) the $K$-mH heatmap.}
\label{Scross}
\end{figure}
\subsubsection{The Cigarette-Bullseye dataset} We have another
\begin{figure}[h]
\centering
\mbox{
\hspace{-0.1in}
\subfloat[$\mathcal{R} = 1.0$, $K$ = 8]{
\includegraphics[angle=0,totalheight=2.3in]{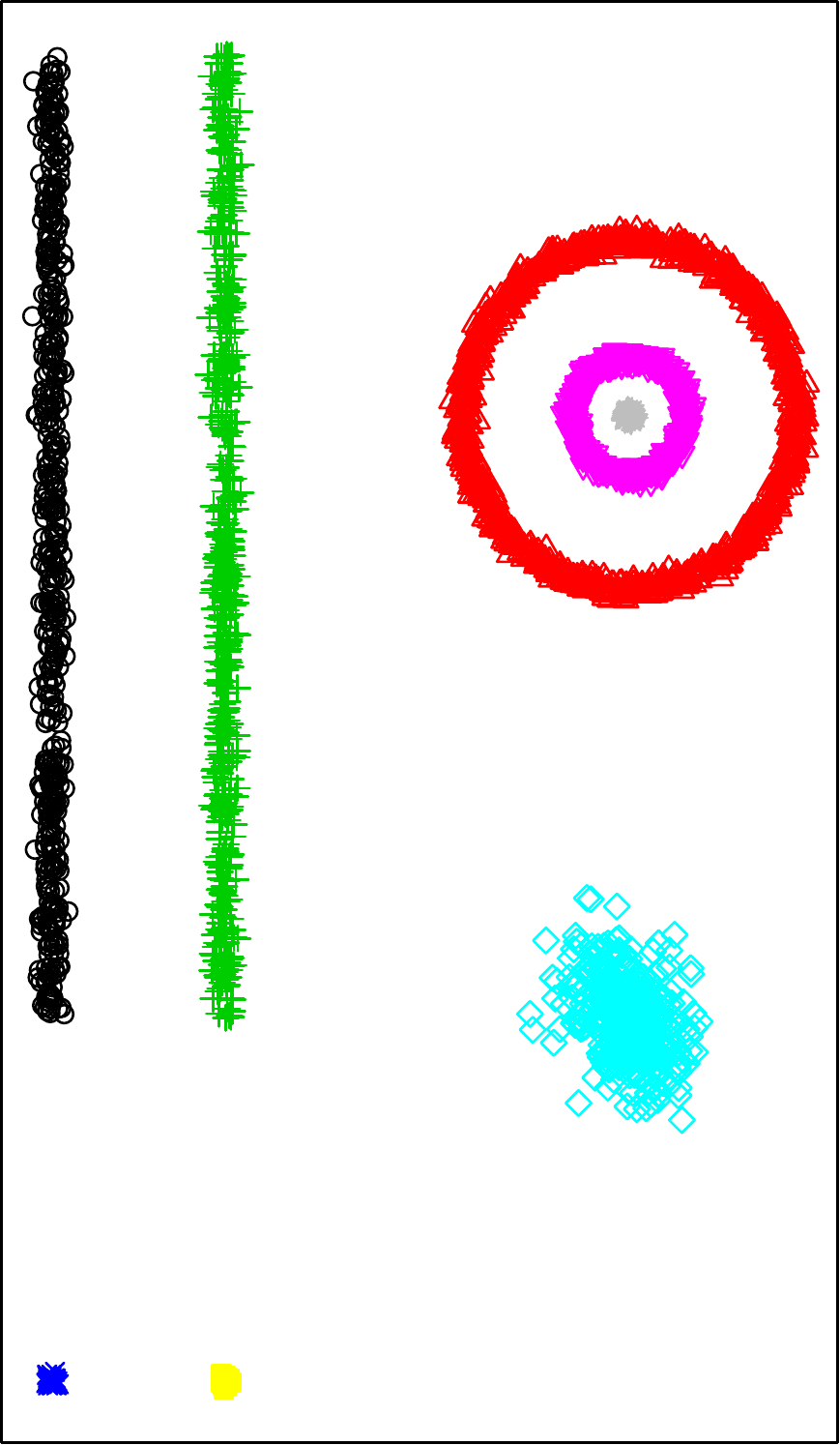}}
\subfloat[$\mathcal{R} = 0.99$, $K$ = 6]{
\includegraphics[angle=0,totalheight=2.3in, width=1.35in]{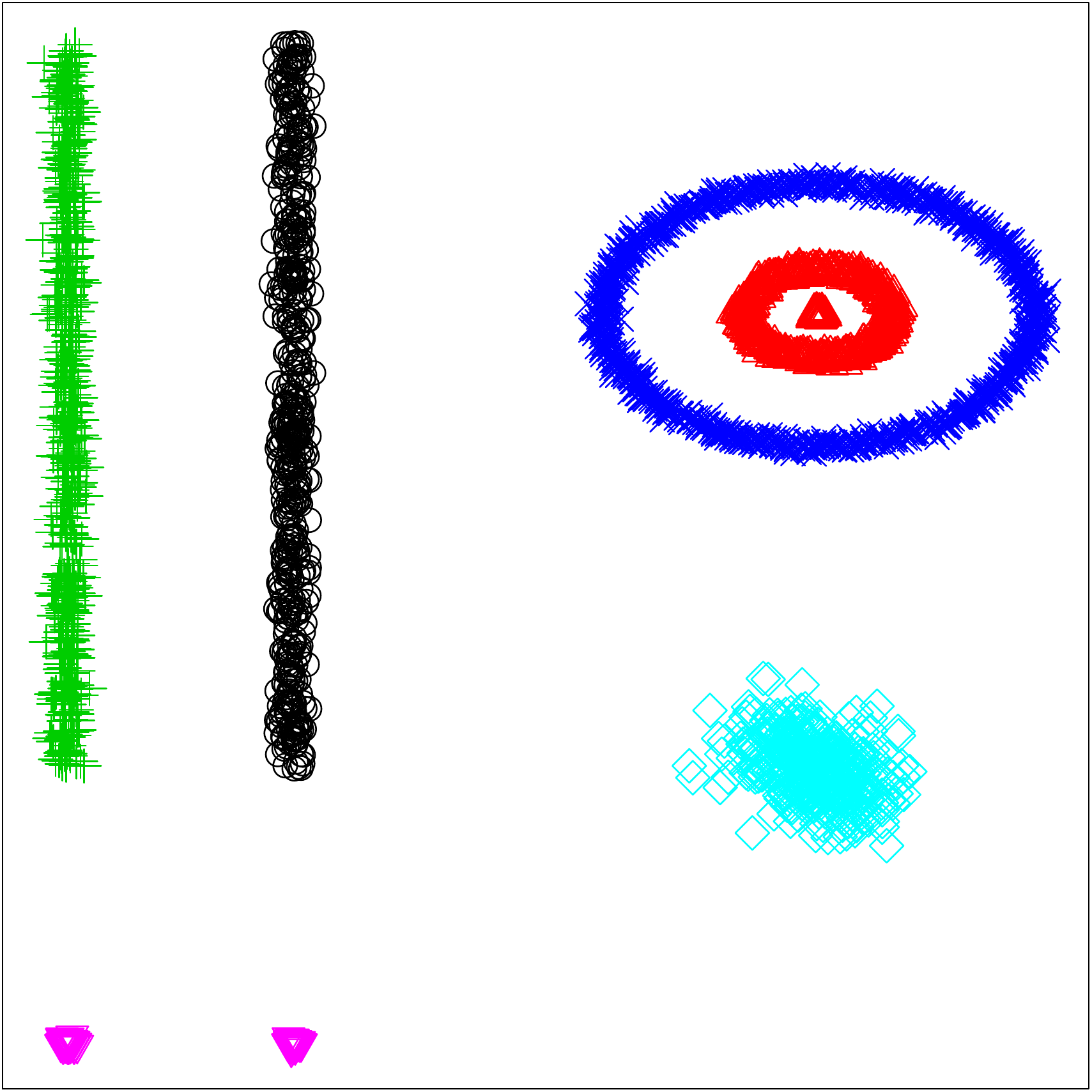}}
\subfloat[$\mathcal{R} = 0.96$, $K$ = 9]{
\includegraphics[angle=0,totalheight=2.3in,width=1.35in]{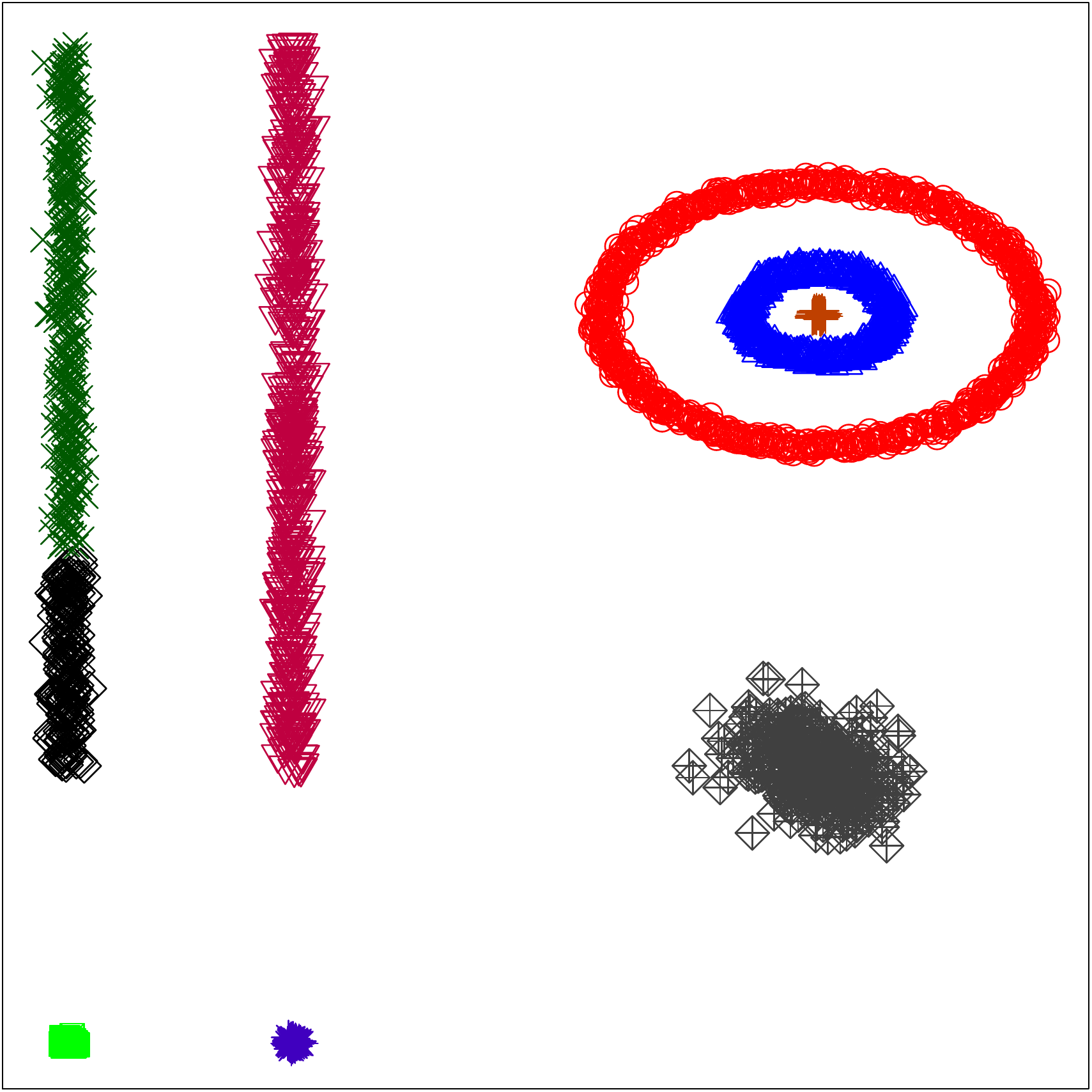}}
\subfloat[$K$-mH heatmap]{
\includegraphics[angle=0,totalheight=2.3in]{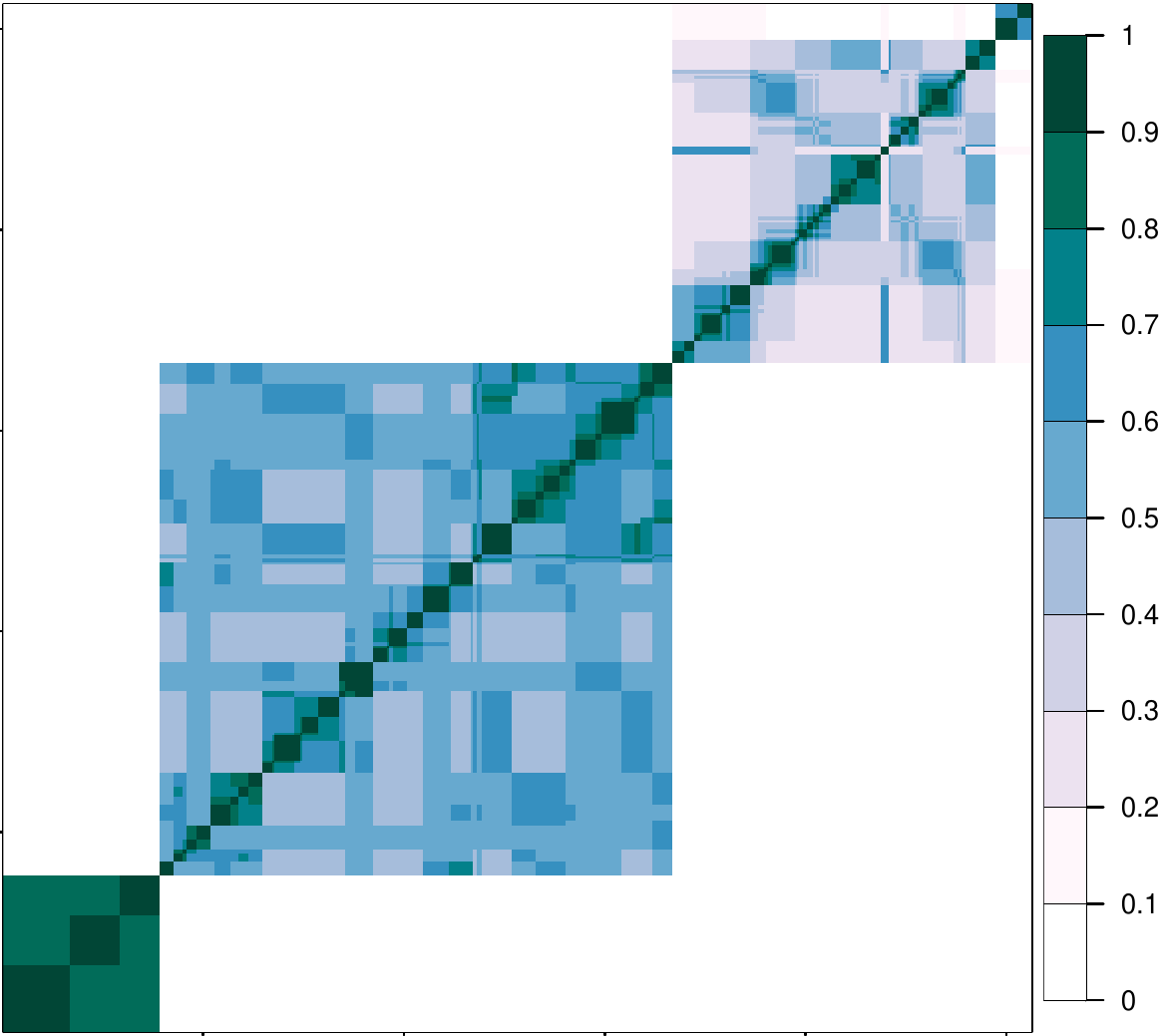}}
}
\caption{Top three performers on the Cigarette-Bullseye dataset: (a)
  $K$-mH, (b) CM, (c) FJ and (d) the $K$-mH heatmap.}
\label{cigbulls}
\end{figure}
 complex-structured dataset with 3 concentric ringed groups, 2 long
groups above 2 small spherical ones and 1 group that is actually a
superset of 2 overlapping Gaussian 
groups. $K$-mH and FJ perform similarly while CM finds 6 clusters but
$\mR=0.99$ because the smaller groups are the ones not identified
clearly. GSL-NN also underestimates the number of groups to be 6, with
$\mR=0.78$. Both DEMP $(\mR = 0.62)$ and DEMP+ $(\mR = 0.64)$ exhibit
poorer performance. The heatmap has similar characteristics as 
SCX, with 3-4 large groups but no clear choice for $K_*$ beyond that
even though there are suggestions of sub-groups within each of the
large groups. However, 
estimates of $K_*$ were 8 (50\% of the time),  9 (42\%) and 10 (8\% of
the time). The median $K_*= 8$ yields  the perfect $K$-mH solution of
Figure~\ref{cigbulls}a while $K_*=9$ breaks the leftmost long cluster
further into two groups, yielding a similar partitioning as FJ (Figure~\ref{cigbulls}c).
\subsection{Higher-Dimensional Datasets}
We next present performance evaluations on three higher-dimensional
datasets often used in the literature.  
\subsubsection{Olive Oils:}
This  dataset~\citep{forinaandtiscornia82,forinaetal83} has
measurements on 8 chemical components for 572 samples of olive
oil taken from 9 different areas in Italy which are from three 
regions: Sardinia and Northern and Southern Italy. 
For this dataset, GSL-NN is the only method that identifies 9
groups ($\mR=0.61$) while FJ identifies 8 groups ($\mR=0.54$). CM ($\mR=0.75$),
DEMP ($\mR=0.82$)and DEMP+ ($\mR=0.85$) are the best ($\mR=0.75$) performers
even though they identify only 7 groups. The visualization step of the $K$-mH
algorithm on the other hand largely identifies 8 kinds of olive
oils (88\% of the time) and also 7 (2\%) and 9 (12\%) kinds of olive oils. The
median estimated $K_*=11$ yields a partitioning with $\mR = 0.67$. A 
closer look at the $K$-mH partitions reveals that oils from the southern
areas of Calabria, Sicily and South Apulia are mainly grouped together in our
second cluster while the remaining southern area of North-Apulia
primarily populates our ninth cluster. Coastal and Inland Sardinian olive oils are identified very well by our
groupings. 
\begin{figure}[h]
   \centering
    \subfloat{
      \raisebox{-.5\height}{\includegraphics[width=0.3\textwidth]{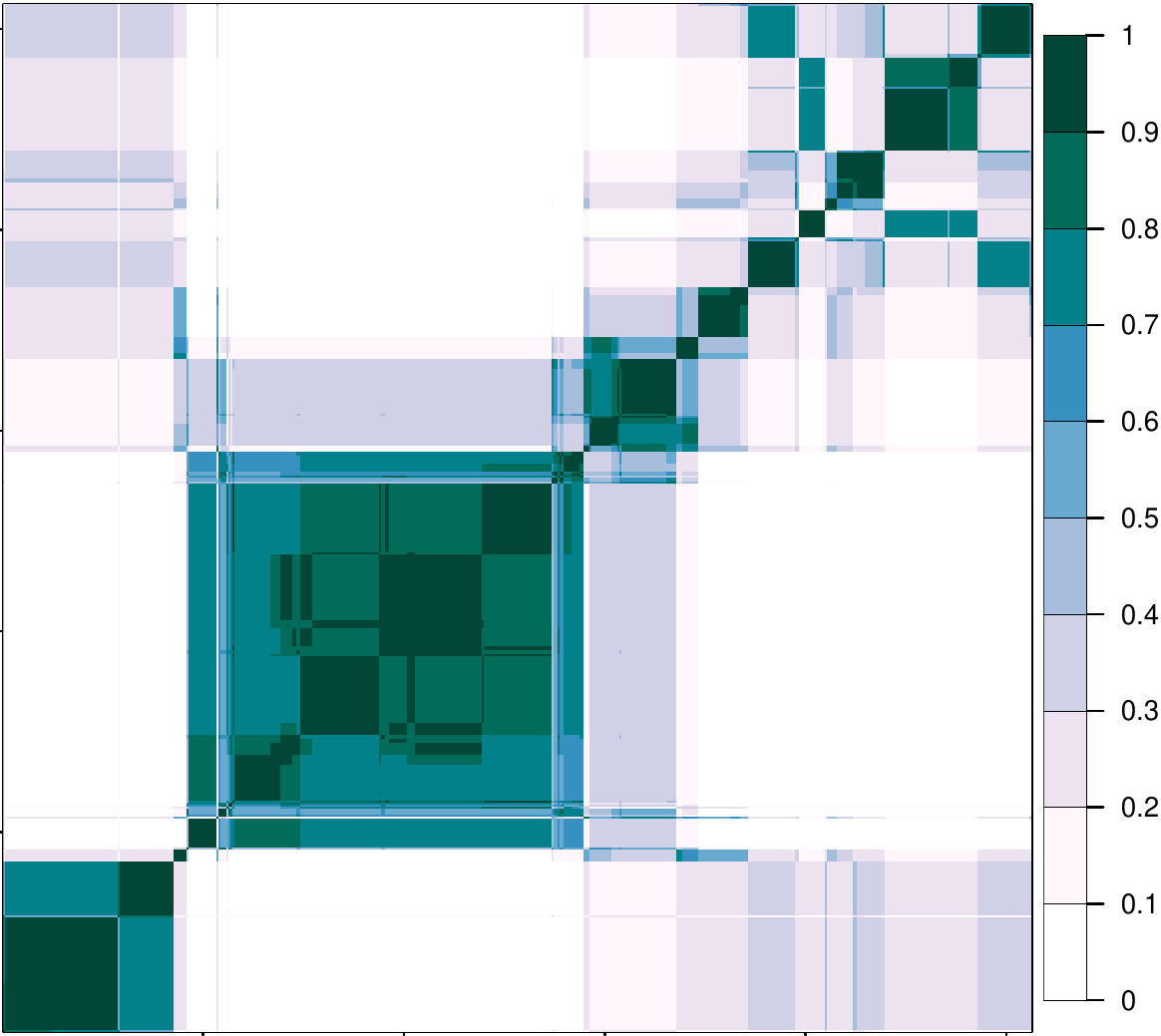}}
    }
    \subfloat{\small
          \centering
\begin{tabular}{ll rrrrrrrrrrr}
  \toprule
Region  &  Area    & \multicolumn{1}{r}{   1} & \multicolumn{1}{r}{   2} & \multicolumn{1}{r}{   3} & \multicolumn{1}{r}{   4} & \multicolumn{1}{r}{   5} & \multicolumn{1}{r}{   6} & \multicolumn{1}{r}{   7} & \multicolumn{1}{r}{   8} & \multicolumn{1}{r}{   9} & \multicolumn{1}{r}{  10} & \multicolumn{1}{r}{  11} \\ 
   \midrule
North    & East Liguria    &   0 &   0 &   4 &   8 &  12 &   0 &   0 &   0 &   0 &   0 &  26 \\ 
           & Umbria          &   0 &   0 &  51 &   0 &   0 &   0 &   0 &   0 &   0 &   0 &   0 \\ 
           & West Liguria    &   0 &   0 &   0 &   0 &  37 &   0 &   0 &   0 &   0 &  13 &   0 \\ 
  Sardinia & Coastal Sardinia  &   0 &   0 &   0 &   0 &   0 &   0 &   3 &  30 &   0 &   0 &   0 \\ 
           & Inland Sardinia &   0 &   0 &   0 &   0 &   0 &   0 &  65 &   0 &   0 &   0 &   0 \\ 
  South    & Calabria        &   0 &  55 &   0 &   0 &   0 &   0 &   0 &   0 &   0 &   0 &   1 \\ 
           & North Apulia    &   0 &   1 &   0 &   0 &   0 &   0 &   0 &   0 &  24 &   0 &   0 \\ 
           & Sicily          &   5 &  17 &   0 &   0 &   0 &   8 &   0 &   0 &   6 &   0 &   0 \\ 
           & South Apulia    &   2 & 204 &   0 &   0 &   0 &   0 &   0 &   0 &   0 &   0 &   0 \\ 
   \bottomrule

   \bottomrule
\end{tabular}
}
     \caption{The $K$-mH heatmap and the results by region and area
       obtained from $K$-mH clustering of the Olive Oils dataset.}
\label{oliveoils}
\end{figure}
Our partitioning aligns very well with the three regions
with  our groups 3, 4, 5, 10 and 11 (with the exception of one oil)
all exclusively from the north, 
groups 7 and 8 from Sardinia and groups 1, 2, 6 and 9 exclusively from the
south. The near-perfect embedding of our 
groups within the three regions indicates that the nine areas drawn using
political geography may not distinguish the different
kinds of olive oils as well as a different characterization using 
a different set of sub-regions that are based on physical geography.
\subsubsection{Zipcode Images:} \label{zip}
The zipcode images dataset made available by \citet{nugent10} has been
used in machine learning to evaluate clustering and classification
algorithms and consists of 2000 $16\times16$ images of handwritten
Hindu-Arabic numerals. Thus, $p=256$ here. \citet{nugent10} report
that GSL-NN ``vaguely'' finds 9 groups ($\mR=0.64$) but that their
10-groups solution is worse ($\mR=0.54$). We normalized the
measurements for each digit to have zero mean and unit variance so
that the Euclidean distance between any two observations is 
negatively but affinely related to the correlation between them. We
reduced dimensions by principal components analysis and used the
projection of the observations into the space spanned by the first
54 principal components which explain at least 90\% of the variation in the
data. This dataset is perhaps too cumbersome for CM, DEMP and DEMP+ while  
FJ finds 6 groups but the assignment is not very far from
random~($\mR=0.05$). 
\begin{figure}[h]
   \centering
    \subfloat{
      \raisebox{-.5\height}{\includegraphics[width=0.5\textwidth]{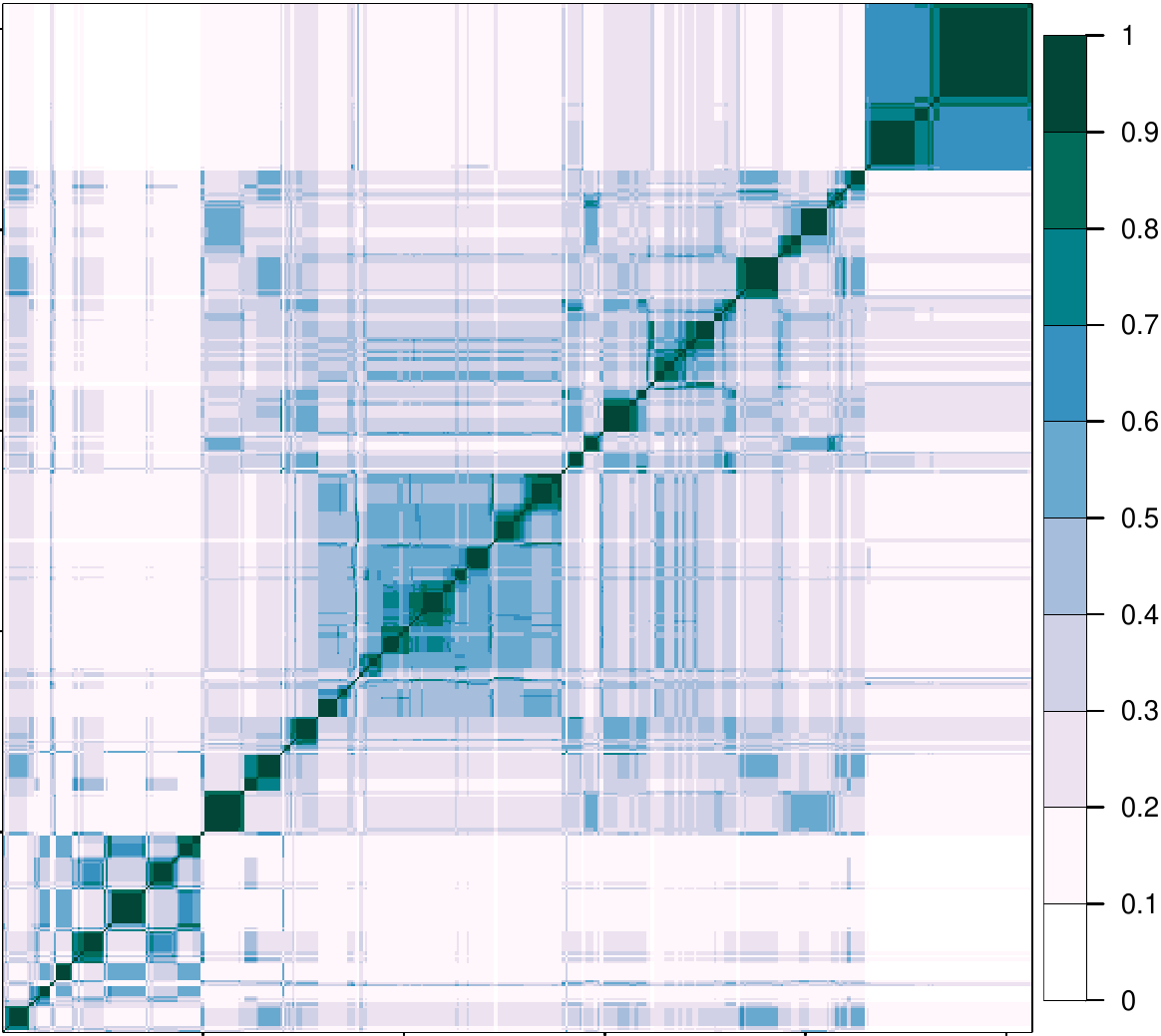}}
    }
    \subfloat{
\tiny
\centering
\begin{tabular}{ll rrrrrrrrrr}
  \toprule
&& \multicolumn{9}{c}{\bf Zipcode Digit}\\
Group & & \multicolumn{1}{r}{   0} & \multicolumn{1}{r}{   1} & \multicolumn{1}{r}{   2} & \multicolumn{1}{r}{   3} & \multicolumn{1}{r}{   4} & \multicolumn{1}{r}{   5} & \multicolumn{1}{r}{   6} & \multicolumn{1}{r}{   7} & \multicolumn{1}{r}{   8} & \multicolumn{1}{r}{   9} \\ 
     &            & \multicolumn{1}{r}{    } & \multicolumn{1}{r}{    } & \multicolumn{1}{r}{    } & \multicolumn{1}{r}{    } & \multicolumn{1}{r}{    } & \multicolumn{1}{r}{    } & \multicolumn{1}{r}{    } & \multicolumn{1}{r}{    } & \multicolumn{1}{r}{    } & \multicolumn{1}{r}{    } \\ 
   \midrule
1  &            &   5 &   0 &   0 &   0 &   0 &   2 &  51 &   0 &   0 &   0 \\ 
  2  &            &   0 &   0 &   0 &   0 &   0 &   0 &  45 &   0 &   0 &   0 \\ 
  3  &            &   0 &   0 &   0 &  62 &   0 &   0 &   0 &   0 &   2 &   0 \\ 
  4  &            &   0 & 323 &   0 &   0 &  19 &   2 &   2 &   2 &   2 &   0 \\ 
  5  &            &  48 &   0 &   1 &   0 &   0 &   2 &   0 &   0 &   1 &   0 \\ 
  6  &            &   0 &   0 &   0 &   1 &  35 &   0 &   0 &   2 &   1 &  10 \\ 
  7  &            &   0 &   0 &   0 &   1 &   0 &  36 &   0 &   0 &   0 &   0 \\ 
  8  &            &   1 &   0 &   1 &   4 &   0 &   2 &   0 &   2 & 115 &   1 \\ 
  9  &            &   0 &   0 &  41 &   0 &   0 &   0 &   0 &   0 &   0 &   0 \\ 
  10 &            &   0 &   0 &   1 &   0 &  73 &   0 &   0 &   0 &   1 &   0 \\ 
  11 &            &   0 &   0 &   0 &   0 &   0 &   0 &  25 &   0 &   0 &   0 \\ 
  12 &            &   1 &   0 &   1 &   6 &   0 &  18 &   0 &   0 &   0 &   0 \\ 
  13 &            & 135 &   0 &   4 &   0 &   0 &   3 &   0 &   0 &   1 &   0 \\ 
  14 &            &   0 &   0 &   0 &   0 &   0 &   0 &   0 &  52 &   0 &   1 \\ 
  15 &            &   1 &   0 &  38 &   0 &   1 &   1 &   0 &   0 &   0 &   0 \\ 
  16 &            &   0 &   0 &   1 &   0 &   4 &   0 &   0 &  12 &   1 &  40 \\ 
  17 &            &   0 &   0 &  57 &   0 &   0 &   0 &   0 &   0 &   0 &   0 \\ 
  18 &            &   0 &   0 &  33 &   0 &   0 &   0 &   0 &   0 &   2 &   0 \\ 
  19 &            &   0 &   0 &   1 &   0 &   0 &   2 &   0 &  19 &   2 &   2 \\ 
  20 &            &  59 &   0 &   1 &   0 &   0 &   1 &   1 &   0 &   0 &   0 \\ 
  21 &            &   1 &   0 &   4 &  56 &   0 &   3 &   0 &   0 &   2 &   0 \\ 
  22 &            &  69 &   0 &   0 &   0 &   0 &   0 &   0 &   0 &   0 &   0 \\ 
  23 &            &   1 &   0 &   0 &   0 &   0 &   0 &  41 &   0 &   1 &   0 \\ 
  24 &            &   0 &   0 &   0 &   1 &   5 &   0 &   0 &   7 &   1 &  59 \\ 
  25 &            &  68 &   0 &   0 &   0 &   0 &   0 &   0 &   0 &   0 &   0 \\ 
  26 &            &   0 &   0 &   0 &   0 &   0 &  29 &   1 &   0 &   1 &   0 \\ 
  27 &            &   0 &   0 &   2 &   0 &   0 &   0 &   0 &  76 &   0 &   5 \\ 
  28 &            &   0 &   0 &  33 &   0 &   0 &   0 &   0 &   0 &   0 &   0 \\ 
  29 &            &   0 &   0 &   1 &  18 &   0 &   0 &   0 &   0 &  24 &   1 \\ 
  30 &            &   0 &   0 &   0 &   0 &   6 &   1 &   0 &  10 &   1 &  49 \\    \bottomrule
\end{tabular}

}
\caption{The $K$-mH heatmap and the results by digit
  obtained from $K$-mH clustering of the Zipcode dataset.}
\label{zipcodes}
\end{figure}
The $K$-mH heatmap~(Figure~\ref{zipcodes}) indicates 
lack of clarity in the number of groups with $K_*$ chosen at between
29 and 30 most of the time. The median $K_*=30$  yields the
grouping ($\mR=0.54$) of Figure~\ref{zipcodes}. Inspection indicates
five main types of handwritten digits for 0 and 2, four kinds for 6,
three kinds of 4 and 9, two major kinds of 3, 5 and 7 and one major kind
for each of 1 and 8. Our groups correspond very reasonably to
handwriting styles for digits and are very interpretable. 
\subsubsection{Handwritten Digits:}
The Handwritten Digits
dataset~\citep{alimoglu96,alimogluandalpaydin96} available
from~\citet{newmanetal98} measured 16 attributes from
250 handwritten samples of 30 writers. With eight samples
unavailable, this dataset has 10992 records. 
We used the first  7 principal component
scores which explained 90\% of the variation in the dataset. We were
only able to apply FJ and $K$-mH (the other methods all threw up
errors). FJ identified 10 groups  
\begin{figure}[h]
   \centering
    \subfloat{
      \raisebox{-.5\height}{\includegraphics[width=0.4\textwidth]{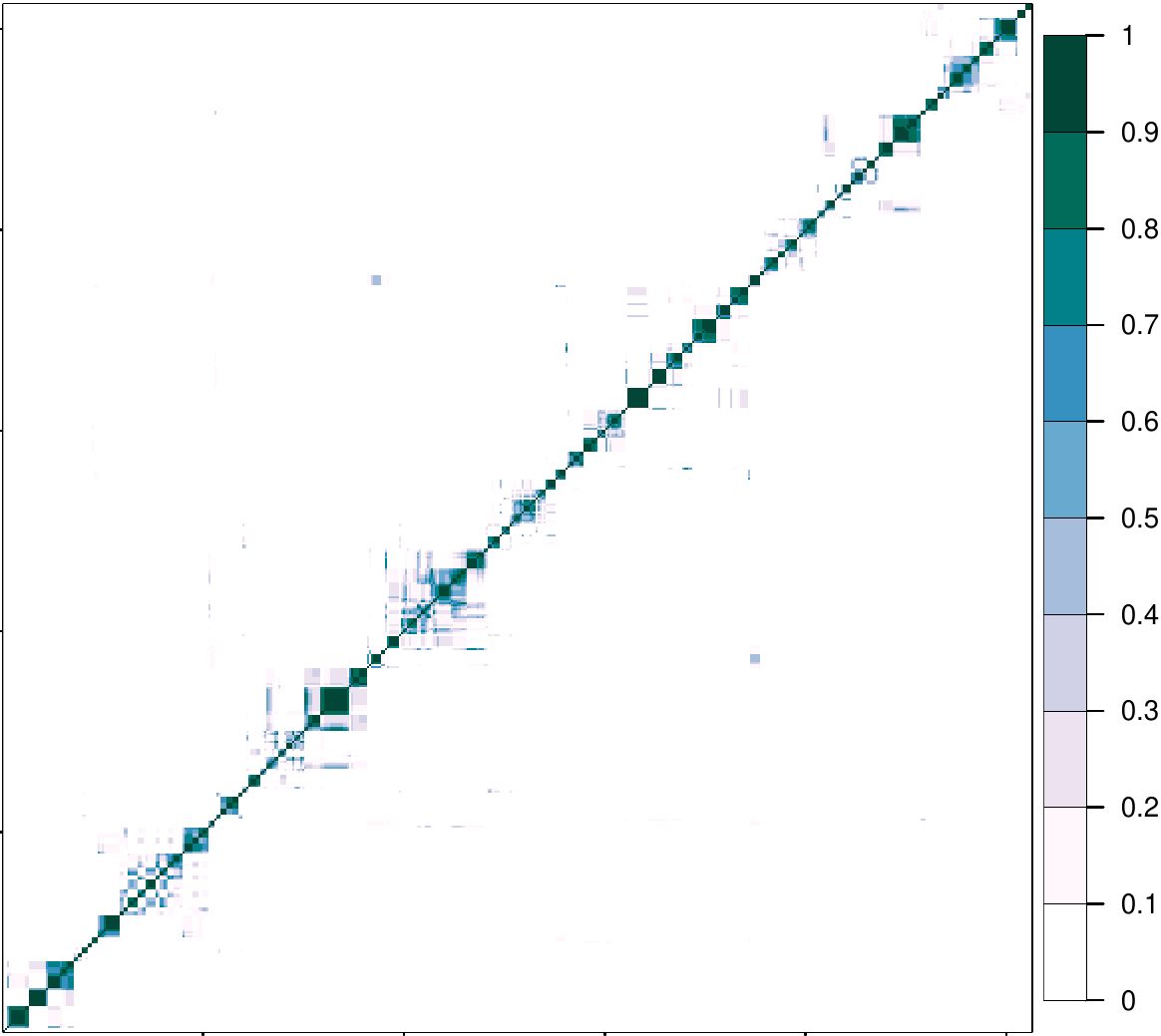}}
    }
    \subfloat{
\tiny
\centering
\begin{tabular}{ll rrrrrrrrrr}
  \toprule
&& \multicolumn{9}{c}{\bf Handwritten Digit}\\
Group & & \multicolumn{1}{r}{   0} & \multicolumn{1}{r}{   1} & \multicolumn{1}{r}{   2} & \multicolumn{1}{r}{   3} & \multicolumn{1}{r}{   4} & \multicolumn{1}{r}{   5} & \multicolumn{1}{r}{   6} & \multicolumn{1}{r}{   7} & \multicolumn{1}{r}{   8} & \multicolumn{1}{r}{   9} \\ 
     &            & \multicolumn{1}{r}{    } & \multicolumn{1}{r}{    } & \multicolumn{1}{r}{    } & \multicolumn{1}{r}{    } & \multicolumn{1}{r}{    } & \multicolumn{1}{r}{    } & \multicolumn{1}{r}{    } & \multicolumn{1}{r}{    } & \multicolumn{1}{r}{    } & \multicolumn{1}{r}{    } \\ 
   \midrule
        1  &           &    2 &  359 & 1140 &    2 &    1 &    0 &    2 &    8 &   70 &    0 \\ 
        2  &           &    0 &    3 &    0 &    0 &    1 &    0 & 1050 &    0 &    0 &    0 \\ 
        3  &           &    0 &    0 &    0 &    0 &    0 &    0 &    0 &   79 &   48 &    0 \\ 
        4  &           &   17 &    2 &    0 &    1 & 1113 &    1 &    2 &    0 &    0 &   13 \\ 
        5  &           &    1 &  513 &    3 &   30 &   15 &  232 &    1 &  143 &   28 &  276 \\ 
        6  &           &   48 &    0 &    0 &    0 &    0 &    0 &    0 &    0 &    1 &    0 \\ 
        7  &           &    0 &    0 &    1 &    0 &    0 &    0 &    0 &  609 &    3 &    0 \\ 
        8  &           &    2 &    0 &    0 &    2 &    4 &    7 &    0 &    0 &    0 &  682 \\ 
        9  &           &  251 &    0 &    0 &    0 &    0 &    0 &    0 &    0 &    1 &    0 \\ 
        10 &           &    0 &  158 &    0 &    3 &    0 &    0 &    0 &    1 &    0 &    1 \\ 
        11 &           &    0 &   31 &    0 & 1017 &    0 &   13 &    0 &    0 &    5 &    5 \\ 
        12 &           &  228 &    0 &    0 &    0 &    0 &    0 &    0 &    0 &    0 &    0 \\ 
        13 &           &    0 &    0 &    0 &    0 &    0 &    3 &    0 &    0 &  408 &    1 \\ 
        14 &           &    1 &    0 &    0 &    0 &    0 &    0 &    0 &    0 &  488 &    0 \\ 
        15 &           &   57 &    0 &    0 &    0 &    0 &    0 &    0 &    0 &    0 &    0 \\ 
        16 &           &    0 &    1 &    0 &    0 &    0 &    0 &    0 &  302 &    1 &    0 \\ 
        17 &           &    0 &    0 &    0 &    0 &    0 &    1 &    0 &    0 &    0 &   38 \\ 
        18 &           &    1 &    0 &    0 &    0 &    0 &    0 &    0 &    0 &    0 &   24 \\ 
        19 &           &    0 &    0 &    0 &    0 &   10 &  174 &    0 &    0 &    0 &   15 \\ 
        20 &           &  416 &    0 &    0 &    0 &    0 &    0 &    0 &    0 &    0 &    0 \\ 
        21 &           &    0 &    0 &    0 &    0 &    0 &  624 &    1 &    0 &    2 &    0 \\ 
        22 &           &    0 &   76 &    0 &    0 &    0 &    0 &    0 &    0 &    0 &    0 \\ 
        23 &           &  100 &    0 &    0 &    0 &    0 &    0 &    0 &    0 &    0 &    0 \\ 
        24 &           &   19 &    0 &    0 &    0 &    0 &    0 &    0 &    0 &    0 &    0 \\ 
   \bottomrule
\end{tabular}

}
\caption{The $K$-mH heatmap and the results by digit
  obtained from $K$-mH clustering of the Handwritten Digits dataset.}
\label{hwcodes}
\end{figure}
but performs very poorly ($\mR=0.097$) while
 the $K$-mH
heatmap~(Figure~\ref{hwcodes}) identifies a range of $K_*=19$ through
27, with a median of 24. The $K$-mH grouping for $K_*=24$
yielded moderately good performance~($\mR=0.64$). Interestingly, our
groups identified 2, 4 and 6 well, but not with a
simpler digit like 1, which, in the light of our findings in
Section~\ref{zip}, may suggest that the 16 attributes  
used to characterize the samples may have focused more on some features
of the handwriting of digits. 
\begin{table}[!h]																																				
\begin{center}																																							
\caption{Performance in terms of $\mathcal{R}$ (first row of
  each block) and estimated
  number of groups $\hat K$ (second row of each block) for all
  datasets used in the   experiment.  A ``-'' indicates 
  that the algorithm failed to converge or returned an error message.}
\label{tab_performance}				
\begin{tabular}{|c|c||c||c|c|c|c|c|c|}						
\hline
\multicolumn{2}{|c||}{Dataset}&\multirow{2}{*}{Measure}&\multicolumn{6}{|c|}{Method}\\
\cline{1-2}\cline{4-9}
Name&$(N,p,K)$&&FJ & CM & GSL-NN & DEMP & DEMP+ & $K$-mH\\
\hline
\hline
\multirow{2}{*}{Banana-Clump}  &\multirow{2}{*}{(200,2,2)} & $\mR$&1.0
              & 0.78&
                                                                   1.0
              & 0.77 & 1.0 & 1.0\\  \cline{3-9}
&&$\hat K$ & 2 & 3 & 2 & 3 & 2 & 2\\ \hline
\multirow{2}{*}{Bullseye}
                              &\multirow{2}{*}{(400,2,2)}&$\mR$&0.99&0.53&0.74&0.21&0.31&0.99\\
  \cline{3-9}
&&$\hat K$ & 2 & 5 & 2 & 7 & 6 & 2\\ \hline
\multirow{2}{*}{Banana-Spheres}
                              &\multirow{2}{*}{(3015,2,3)}&$\mR$&0.95&0.53&0.74&0.29&0.45&0.99\\
  \cline{3-9}
&&$\hat K$ & 5 & 11 & 2 & 18 & 13 & 3\\ \hline
\multirow{2}{*}{SCX}
                              &\multirow{2}{*}{(3420,2,8)}&$\mR$&0.89&0.78&0.53&0.77&0.78&1.0\\
\cline{3-9}
                              &&$\hat K$ & 8 & 12 & 7 & 12 & 12 & 8\\ \hline
 \multirow{2}{*}{Cigarette-Bullseye}
                              &\multirow{2}{*}{(3025,2,8)}&$\mR$&0.96&0.99&0.78&0.62&0.64&1.0\\
\cline{3-9}
                              &&$\hat K$ & 9 & 6 & 6 & 11 & 10 & 8\\ \hline
 \multirow{2}{*}{Olive Oils}
                              &\multirow{2}{*}{(572,8,9)}&$\mR$&0.54&0.75&0.61&0.82&0.85&0.67\\
\cline{3-9}
                              &&$\hat K$ & 8 & 7 & 9 & 7 & 7 & 11\\ \hline
  
 \multirow{2}{*}{Zipcode Digits}
                              &\multirow{2}{*}{(2000,256,10)}&$\mR$&0.05&-&0.64&-&-&0.54\\
\cline{3-9}
                              &&$\hat K$ & 8 & - & 9 & - & - & 26\\ \hline
  \multirow{2}{*}{Handwritten Digits}
                              &\multirow{2}{*}{(10992,16,10)}&$\mR$&0.10&-&-&-&-&0.64\\
\cline{3-9}
                              &&$\hat K$ & 10 & - & - & - & - & 24\\
  \hline\hline \hline
 \multicolumn{3}{|l|}{Number of cases where a competitor performs better} &
                                                                      6
                                                       & 8 & 7 & 8 & 6
                        & 2 \\ \hline
\end{tabular}
\end{center}
\end{table}

The performances of $K$-mH, FJ, CM, DEMP and DEMP+ for all cases are 
summarized  in Table~\ref{tab_performance} and shows
that $K$-mH is always among the top performers. This happens with very
complicated as well as simpler structures. Even when performance
is not outstanding, as happened with
higher-dimensional real-life datasets, $K$-mH is still a top
performer, often producing results that are interpretable. FJ is also a good
performer in the two-dimensional examples but this performance
degraded more in higher dimensions than with $K$-mH. DEMP ad DEMP+ was
a good performer only on the Olive Oils dataset where it performed
very well despite underestimating the number of groups by 2. Our
algorithm was also able to handle computations for
the larger handwritten digits dataset.

\section{Discussion}\label{dis}

In this paper we propose a new $K$-means hierarchical clustering
algorithm that builds on the idea of
\citet{fred05} that different
clusterings of a dataset each provide different discrete evidence of a
grouping. 
 We compare several different clusterings of the data and
choose the final grouping that is most similar to the  proposed
partitions. 
Our algorithm is among the top performing methods for both simulated
datasets with complicated shapes as well as several real datasets. We
also present an automated clustering approach for finding the optimal
parition and  number of groups that is shown to perform well.     In
addition, we use a graphical method introduced in \citet{fred05} that
we use to investigate uncertainty and structural stability of the
clustering and to determine the correct number of groups.
Our $K$-mH algorithm is computationally efficient for larger
datasets in comparison to several other cluster merging
algorithms. Indeed, the main computational cost is that of performing
$K$-means for different $K$, which can be expensive given the number
of initializing runs for each $K$. Further, it is very easily coded:
simple {\tt R} functions doing the same are available on request.

There are several  directions for future work. One possibility  is
to compare other distance measures in the hierarchical step of the
$K$-mH algorithm. It may be worthwhile to further use
other different distance measures as candidate partitions when
choosing the optimal partition $P_*$. Another aspect worthy of
investigation  would be to explore additional ways for determining
$K_*$. It is worth noting in this context that
the hierarchical map for visualizing structural stability can be a
memory-intensive operation.  Thus, we see that while we have put
forward a promising algorithm, issues meriting further attention remain.

\ack{ This research was supported, in part, by National Science
Foundation (NSF) grants DMS-0707069, DMS-CAREER-0437555 and by the
National Institutes of Health grant R21EB0126212. The content of this
paper however is solely the responsibility of the authors and does not
represent the official views of the NSF or the NIH.}

\bibliographystyle{wb_stat}
\bibliography{referencespaper3}

\end{document}